\pdfoutput=1
\documentclass[10pt,a4paper,oneside]{article}
\usepackage{bbding}

\usepackage{microtype}
\usepackage{graphicx}
\usepackage{subfigure}
\usepackage{booktabs} 
\usepackage{multirow}
\usepackage{stfloats}
\usepackage{float}
\usepackage{amsmath, bm}
\usepackage[T1]{fontenc}
\usepackage{xcolor}
\usepackage{amsthm}
\usepackage{bookmark} 
\usepackage{array} 
\usepackage{xcolor} 
\usepackage{algorithm}
\usepackage{algorithmic}
\usepackage{url}
\usepackage{hyperref}
\usepackage{setspace}

\usepackage{amsmath}
\usepackage{amssymb}
\usepackage{mathtools}
\usepackage{amsthm}
\usepackage[numbers,square,sort&compress]{natbib}
\usepackage[capitalize,noabbrev]{cleveref}

\theoremstyle{plain}
\newtheorem{theorem}{Theorem}[section]

\newtheorem{lemma}[theorem]{Lemma}

\theoremstyle{definition}
\newtheorem{definition}[theorem]{Definition}

\theoremstyle{remark}
\newtheorem{remark}[theorem]{Remark}

\begin{document}

\title{\bf{Sequential Least-Squares Estimators with Fast Randomized Sketching for Linear Statistical Models}
	\thanks{First author: Guan-Yu Chen (chenguanyuu@nuaa.edu.cn); Corresponding author: Xi Yang (yangxi@nuaa.edu.cn).}}

\author{
	Guan-Yu Chen\thanks{School of Mathematics, Nanjing University of Aeronautics and Astronautics, Nanjing 211106, China.}
	, Dong-Yue Xie\footnotemark[2]
	, Xi Yang\footnotemark[2]
}
\maketitle
\begin{center}
\begin{abstract}
	We propose a novel randomized framework for the estimation problem of large-scale linear statistical models, namely Sequential Least-Squares Estimators with Fast Randomized Sketching (SLSE-FRS), which integrates \textit{Sketch-and-Solve} and \textit{Iterative-Sketching} methods for the first time. By iteratively constructing and solving sketched least-squares (LS) subproblems with increasing sketch sizes to achieve better precisions, SLSE-FRS gradually refines the estimators of the true parameter vector, ultimately producing high-precision estimators. We analyze the convergence properties of SLSE-FRS, and provide its efficient implementation. Numerical experiments show that SLSE-FRS outperforms the state-of-the-art methods, namely the Preconditioned Conjugate Gradient (PCG) method, and the Iterative Double Sketching (IDS) method.
	\bigskip
	
	\noindent{\bf Keywords:} least-squares estimation, randomized sketching, Sketch-and-Solve, Iterative-Sketching
	
	\noindent{\bf MSC codes:} 65F10, 65F20, 65K05, 68W20
\end{abstract}
\end{center}

\section{Introduction}
\label{Introduction}
Linear regression analysis is one of the most classic and fundamental methods used to describe the relationships between variables. Suppose that there exists a standard linear relationship between the response vector \( Y \in \mathbb{R}^{N}\) and the feature matrix \( X \in \mathbb{R}^{N \times d}\) with the sample size $N$ and the feature size $d$ as follows
\begin{equation}\label{linear_model}
	Y = X \beta + \zeta,  \nonumber
\end{equation}
where \( \beta \in \mathbb{R}^d \) is the unknown true parameter vector to be estimated and \( \zeta \in \mathbb{R}^{N}\) represents the random noise vector with zero mean and a covariance matrix \(\sigma^2I_N \).

To learn the unknown parameter \( \beta \), we consider the ordinary least-squares (OLS) estimator \(\hat{\beta}\),
\begin{equation}\label{main1}
	\hat{\beta} = \underset{\beta \in \mathbb{R}^d}{\arg \min}\ f(\beta;X,Y),
\end{equation}
with \(f(\beta;X, Y) = \frac{1}{2}\|Y-X\beta\|^{2}_{2}\). Throughout the paper, we assume that \( X \) has full column rank, then the OLS estimator can be explicitly formulated as
\begin{equation}
	\hat{\beta} = (X^{\top}X)^{-1}X^{\top}Y.  \nonumber
\end{equation}
Due to its numerous well-established and favorable statistical properties \citep{chatterjee2009sensitivity}, the OLS estimator is widely adopted to estimate the parameter \(\beta\) in practice. However, for large-scale problems with \(N \gg d\), the direct computational complexity \( O(Nd^2) \) to obtain \(\hat{\beta}\) becomes prohibitive. To address this challenge, numerous randomized algorithms based on \textit{sketching} methods have been developed to obtain an approximation of the OLS estimator efficiently.

The first classical randomized line to reduce the computational cost, known as \textit{Sketch-and-Solve} \citep{drineas2011faster,sarlos2006improved}, is using \textit{sketching} matrix \(S \in \mathbb{R}^{m \times N}\) with \(m \ll N\) to construct the \textit{sketched data} \((SX, SY)\) of the original data \((X, Y)\). Rather than solving problem (\ref{main1}) for the OLS estimator from large-scale data \((X,Y)\), one can solve the following \textit{smaller} sketched LS problem to obtain the sketched LS estimator \(\tilde{\beta}\) as an approximation,
\begin{equation}
	\tilde{\beta} = \underset{\beta \in \mathbb{R}^d}{\arg \min}\ f(\beta;SX, SY),
	\nonumber
\end{equation}
with \(f(\beta;SX, SY) = \frac{1}{2}\|SY-SX\beta\|^{2}_{2}\). Then the direct methods can be called to compute the sketched estimator \(\tilde{\beta}\) within \(O(md^2)\) time.

As the \textit{suboptimality} illustrated in \citep{pilanci2016iterative}, any \textit{Sketch-and-Solve} methods based on only observing a single pair of sketched data \((SX, SY)\), unless the sketch size \(m \geq N\), necessarily has a substantially larger error than the OLS estimator.  In other words, with a small sketch size \(m\), the \textit{Sketch-and-Solve} methods result in estimators with relatively low precision.

The second widely adopted line is \textit{Iterative Sketching}, which involves repeatedly sketching the problem and iteratively refining the estimator. Recently, an \textit{Iterative Sketching} method has attracted significant attention, the Iterative Hessian Sketch (IHS) \citep{pilanci2016iterative} and its variants. IHS is an effective and efficient iterative sketching method for large-scale LS problems, which uses refreshed sketched Hessian matrix \(H_t = X^{\top}S_t^{\top}S_tX\) to approximate the Hessian matrix \(H = X^{\top}X\) of (\ref{main1}). The update formula can be expressed as:
\begin{equation}
	\beta_{t+1} = \beta_{t} - H_t^{-1} \nabla f({\beta}_{t}; {X}, {Y}),  \nonumber
\end{equation}
where the sketching matrices \(S_0,\ldots, S_t,\ldots\) are independent and identically distributed (i.i.d.) of size \(m \times N\), with \(m \ll N\) and \(\nabla f({\beta}; {X}, {Y}) := X^{\top}(X\beta - Y)\). The theoretical results in \citep{derezinski2024recent} guarantee that, with high probability, the prediction error can decay at a constant rate and the output can serve as a high-precision estimator for \(\beta\).

In \citep{ozaslan2019iterative}, the convergence rate of IHS is significantly improved by incorporating a momentum term, leading to the momentum iterative Hessian sketch (M-IHS) method. M-IHS uses a fixed sketching matrix $\hat{S} \in \mathbb{R}^{m \times N}$ to approximate the Hessian matrix, denoted as \(\hat{H} = X^{\top}\hat{S}^{\top}\hat{S}X\). This approach avoids the repeated construction of the inverse of the Hessian sketch \(H_t^{-1}\). The update formula can be represented as:
\begin{equation}\label{10}
	\beta_{t+1} = \beta_{t} - \mu \hat{H}^{-1} \nabla f({\beta}_{t}; {X}, {Y}) + \eta(\beta_{t} - \beta_{t-1}), \nonumber
\end{equation}
Based on Marchenko-Pastur law, \citet{ozaslan2019iterative} investigated the optimal step sizes \(\mu\) and \(\eta\), achieving a data-independent convergence rate $(d/m)^{1/2}$. See more related work in \citep{tang2017gradient,lacotte2021adaptive,lacotte2020optimal,lacotte2021faster,na2023hessian,epperly2024fast}

Inspired by the results in \citep{dobriban2019asymptotics}, we noticed that the asymptotic precision of the sketched LS estimator of the \textit{Sketch-and-Solve} method can be explicitly formulated by a function of the sketch size \(m\), the sample size \(N\) and the feature size \(d\). For a fixed sketching matrix (e.g., Gaussian or SRHT), increasing \(m\) improves the precision of the estimator but also raises the computational costs. This theory appears to suggest that we can only strike a balance between improving accuracy and reducing computational complexity. In fact, it provides us with an opportunity to enhance the precision while simultaneously decreasing the computational costs.

We suggest applying the \textit{Sketch-and-Solve} method multiple times with a carefully constructed sequence of sketched LS subproblems with increasing sketch sizes. The precision of the estimators can be iteratively improved. Each sketched LS subproblem can be solved using any efficient iterative LS solver. Importantly, compared to directly applying the solver to the original problem, the cost of performing iterations in the sketched LS subproblems is significantly cheaper. If the solution of an appropriate precision for each subproblem can be obtained at a relatively low cost, this new idea will lead to a substantial reduction in the overall computational expense.

Therefore, we propose a novel framework, named Sequential Least-Squares Estimators with Fast Randomized Sketching (SLSE-FRS). SLSE-FRS repeatedly applies the \textit{Sketch-and-Solve} method with increasing sketch size to compute the estimators for the unknown parameter vector \(\beta\). To the best of our knowledge, this concept is proposed for the first time in this area. However, three key issues need to be addressed.

The first issue is how to construct the sequence of the sketched LS subproblems effectively and efficiently. Since the sketched LS subproblems are constructed using the sketched data $(S_iX, S_iY)$ for $i = 1, 2, \ldots$, where \(S_i \in \mathbb{R}^{m_{i}\times N}\) is the sketching matrix of the sketch size \(m_i\). Constructing each sketched LS subproblem necessitates accessing the original data once, resulting in a computational cost of at least $O(Nd)$. If we independently construct them, this cost becomes unacceptable. Moreover, selecting an appropriate sketch size is crucial, as it directly impacts both the precision and the computational cost.

The second issue is how to determine the stopping criterion for each sketched LS subproblem. Due to Theorem 1 in \citep{pilanci2016iterative}, the error between the sketched LS estimator and the true parameter vector \(\beta\) has a lower bound, limiting achievable precision. In each sketched subproblem, we only need to achieve this level of precision. In practice, the true parameter vector $\beta$ is unobservable. This theoretical quantity is not an available stopping criterion for the iterations. Thus, it is necessary to develop a theoretically rigorous and computationally feasible surrogate to serve as the appropriate stopping condition for each sketched LS subproblem.

The final issue is to ensure theoretically and numerically that SLSE-FRS achieves the same level of precision as the OLS estimator (the noise level or $\sigma$-level).

\subsection{Contributions}
Motivated by the above considerations, we develop the SLSE-FRS framework and address the three issues discussed above in this paper. The main contributions of this paper are summarized as follows.

First, we present a detailed introduction to the SLSE-FRS framework for large-scale linear statistical models. To the best of our knowledge, this is the first framework that systematically unifies \textit{Sketch-and-Solve} with \textit{Iterative-Sketching} in a sequential manner. Fundamentally, SLSE-FRS follows a recursive ``small-sketch warmstart $+$ big-data polish'' paradigm. It constructs a sequence of sketched LS subproblems with increasing sketch sizes, solves each subproblem to its appropriate intrinsic precision, and uses the resulting estimator as the initialization for the next larger subproblem. This recursive initialization process produces a sequence of progressively refined estimators, followed by a final full-scale refinement stage to reach the OLS-level prediction precision. In this way, most iterations are performed on smaller sketched subproblems, while only a very small number of full-scale refinement steps are required, thereby substantially reducing the overall computational cost.

Second, we develop an efficient implementation based on the SLSE-FRS framework. Instead of independently constructing each sketched LS subproblem from the original data, we first construct the largest SRHT-based sketched data and then form the sequence of sketched LS subproblems by extracting row subsets of different sizes from it, avoiding repeated access to the original data. We further adopt M-IHS as the inner LS solver, which uses a fixed sketched Hessian as a randomized preconditioner and avoids repeated Hessian sketch inversions, thereby further improving the computational efficiency of SLSE-FRS. We emphasize that the SLSE-FRS framework is highly flexible: the inner LS solver, the sketching operator, and the sketch-size schedule can all be adjusted within the same sequential structure, making it possible to incorporate more suitable and efficient solvers or sketching strategies for problems with different structures. In fact, we have also developed other variants within this framework, such as SLSE-PCG, which exhibits similarly efficient performance, although this variant is beyond the scope of this paper.

Third, we establish the convergence theory for the efficient SLSE-FRS algorithm. After deriving the convergence bound for the M-IHS inner iterations in each sketched LS subproblem, we establish an overall error bound for the two-stage procedure. This analysis shows that the proposed algorithm can eventually achieve the same level of prediction precision as the OLS estimator. Based on the theoretical analysis, we further derive interpretable lower bounds for the inner iteration numbers \(a_i\), which lead to an explicit complexity estimate for SLSE-FRS. Moreover, the proof strategy and analytical approach provide a theoretical template for analyzing other variants derived from the SLSE-FRS framework.

Finally, we conduct numerical experiments to demonstrate the efficiency of SLSE-FRS. The experiments show that only a small number of inner iterations is needed for the sketched LS subproblems and confirm that SLSE-FRS reaches the OLS-level prediction precision. Compared with IDS, PCG, and M-IHS, SLSE-FRS exhibits faster convergence while significantly reducing the computing time. In addition, the experiments on sketch-size tuning and CountSketch demonstrate the flexibility of the proposed framework and its potential for faster convergence and higher computational efficiency. Related numerical results further indicate that the LS-subproblem sizes can be treated as tunable hyperparameters, and well-chosen sizes can improve computational efficiency. Therefore, learning better sketch-size parameters is a valuable topic for future research, but it is not pursued in this paper.

\subsection{Additional related work}\label{relatedworks}
For the problem (\ref{main1}), the summaries of the classical randomized sketching methods can be found in \citep{woodruff2014sketching,drineas2016randnla,martinsson2020randomized, derezinski2024recent} and references therein. In sketching methods, an appropriate sketching matrix is crucial, as it directly impacts the precision, the computational efficiency, and the stability of these methods. One of the most classical sketching matrix is a matrix \(S \in \mathbb{R}^{m \times N}\) with i.i.d. Gaussian entries \(\mathcal{N}(0, m^{-1})\), see \citep{indyk1998approximate}. Despite its simple structure, generating \(SX\) requires in general \(O(mNd)\) time, which is computationally more expensive than directly solving the problem (\ref{main1}).

Another popular and well-developed orthogonal sketching matrix \(S \in \mathbb{R}^{m \times N}\) is the subsampled randomized Hadamard transform (SRHT), see \citep{sarlos2006improved,2009The,drineas2006sampling}, which is based on the Walsh-Hadamard transform. Its recursive nature allows for efficiently computing the sketched data \((SX, SY)\) in \(Nd\log_2 N\) time, see \citep{woodruff2014sketching}. To further improve the computational efficiency, \citep{clarkson2017low} introduced a \textit{sparse projection} matrix, i.e., the CountSketch matrix $S \in \mathbb{R}^{m \times N}$. The columns of \(S\) are independent, and every column contains only one non-zero element. Due to its sparsity, the sketched data \((SX, SY)\) can be obtained within \(O(Nd)\) time.

\subsection{Organization}
The remainder of the paper is organized as follows. In \cref{sec2}, we introduce the general framework of SLSE-FRS. In \cref{sec:effcntSLSE}, we develop an efficient implementation of SLSE-FRS based on SRHT sketching operators and the M-IHS method. In \cref{sec:theorem}, we concern the convergence guarantee of SLSE-FRS and provide the choices of the inner iteration numbers together with the complexity analysis. Numerical experiments are given in \cref{sec:experiment}. \cref{sec:conclusion} concludes the paper with several remarks and future research directions.

\subsection{Notation}
Throughout the paper, we denote by \( \|\beta\| := \|\beta\|_2 \) the Euclidean norm of a vector \( \beta \). For a matrix \(G\), \(\|G\|\), \(\rho(G)\), and \(G^\top\) denote its spectral norm, spectral radius, and transpose, respectively. The identity matrix of size \(d\) is denoted by \(I_d\). The expectation \(\mathbb{E}[\cdot]\) is taken with respect to the random noise \(\zeta\), while the probabilities associated with sketching are taken with respect to the randomness of the corresponding sketching matrices. All logarithms are natural logarithms, and \(e\) denotes the base of the natural logarithm.

\section{The SLSE-FRS framework}\label{sec2}
In this section, we will introduce the general framework of SLSE-FRS, a new iterative method to repeatedly apply the \textit{Sketch-and-Solve} methods to obtain the estimators for the true parameter vector \( \beta \).

The inspiration comes from \citep{dobriban2019asymptotics}, which analyzed the limits of precision loss incurred by the popular \textit{Sketch-and-Solve} methods. We consider one of the loss functions, named the \textit{relative prediction efficiency} (\text{PE}).
\begin{equation}
	\text{PE}=\frac{\mathbb{E}\|X \tilde{\beta}-X \beta\|^{2}}{\mathbb{E}\|X \hat{\beta}-X \beta\|^{2}}.
\end{equation}
PE measures the precision loss between the sketched LS estimator \(\tilde{\beta}\) and the OLS estimator \(\hat{\beta}\) due to the sketching in the \textit{Sketch-and-Solve} methods. Given data \(X\), PE depends on the sketch size \(m\), the sample size \(N\), and the feature size \(d\). PE decreases with an increasing sketch size \(m\), yielding a higher-precision estimator. This naturally motivates the idea of constructing a sequence of sketched LS subproblems with increasing sketch sizes to compute estimators with progressively higher precisions.

Popular LS solvers iteratively solve the original problem (\ref{main1}). In contrast, we suggest to iteratively solve a sequence of relatively small-scale sketched LS subproblems. Our framework proceeds in two stages. The 1st stage involves constructing and iteratively solving the sketched LS subproblems to compute estimators with progressively higher precisions. The solution of one sketched LS subproblem is used as the initial guess for the next. If the iterates can follow a sufficiently accurate path toward the true parameter vector \(\beta\), our method can greatly reduce the computational cost. In the 2nd stage, we solve the full-scale LS problem to refine the estimator to the OLS-level precision. Since we already have an approximate estimator with a reasonable level of precision at the 1st stage, the 2nd stage can be completed within a few iterations at a very low cost. The goal is to achieve the precision of the OLS estimator with the lowest computational cost.

The new framework can be regarded as an inner-outer iteration method, where solving the sequence of LS subproblems constitutes the outer iteration, and the iteration for each individual LS subproblem forms the inner iteration. In detail, we construct \( K \) sketched LS subproblems,
\begin{equation}\label{sketchedproblem}
	\min \limits_{\beta \in \mathbb{R}^d}\ \frac{1}{2}\|S_{i}X \beta - S_{i}Y \|^{2} , i = 1, \ldots, K,
\end{equation}
where \(S_i \in \mathbb{R}^{m_i \times N}\) is the \(i\)-th sketching matrix and \( (\tilde{X},\tilde{Y}) := (S_{i}X,S_{i}Y) \) is the sketched data. The \(i\)-th sketched LS estimator can be represented as
\begin{equation}
	\tilde{\beta}^i = (\tilde{X}^{\top}\tilde{X})^{-1}\tilde{X}^{\top}\tilde{Y}.
\end{equation}
The sketched LS estimator \(\tilde{\beta}^i\) can be viewed as an approximate estimator for the true parameter vector \(\beta\). Therefore, we consider achieving the same level of precision of \(\tilde{\beta}^i\) when an iterative LS solver is applied to the sketched LS subproblem (\ref{sketchedproblem}). By defining the expected prediction error of the \(i\)-th exact sketched LS estimator \(\tilde{\beta}^i\) relative to \( \beta \) as below
\begin{equation}
	\delta_i := \mathbb{E}\|X(\tilde{\beta}^i - \beta)\|,
\end{equation}
for the \(i\)-th LS subproblem, we hope the adopted iterative LS solver can return an estimator ${\beta}^{i}$ satisfying
\begin{equation}\label{11}
	\mathbb{E}\|X({\beta}^{i} - \beta)\| < (1+\omega) \delta_i,
\end{equation}
where $\omega\in(0,1)$ represents a prescribed tolerance, which means that the returned estimator is required to attain a prediction precision comparable to that of the exact sketched LS estimator in the \(i\)-th LS subproblem.

At this point, we have completed the 1st stage, and obtained $K$ estimators \(\{\beta^i\}_{i=1}^K\) with progressively higher precisions. Since we aim to achieve the precision of the OLS estimator, due to the suboptimality of the \textit{Sketch-and-Solve} method, we have to move on to the 2nd stage, namely apply an iterative LS solver to the full-scale LS problem. Therefore, we utilize the estimator \(\beta^K\) from the iterative solution of the \(K\)-th sketched LS subproblem as the initial guess, which is a high-quality estimator for \(\beta\), only a few additional iterations are required to achieve the OLS precision, leading to significant cost savings.

In the SLSE-FRS framework, we only need to ensure that the iterative solution error of the \(i\)-th sketched LS subproblem reaches the order of \(\delta_i\) for \(i=1,\ldots,K\). Therefore, SLSE-FRS is not limited to any specific LS iterative solver but is compatible with all efficient LS iterative solvers. This compatibility is very powerful as it allows for the seamless integration of any current and future efficient iterative LS solvers within SLSE-FRS, thereby providing users with greater flexibility. Additionally, compared to directly using LS iterative solvers on the original full-scale LS problem, SLSE-FRS significantly reduces computational costs by applying iterative LS solvers on the small-scale sketched LS subproblems. We summarize the general SLSE-FRS framework in Algorithm \ref{Alg-framework}, where \(a_i\) represents the number of iterations performed for the \(i\)-th sketched LS subproblem to meet the stopping criterion, namely the condition (\ref{11}).

\begin{algorithm}[t]
	\caption{The SLSE-FRS framework}
	\label{Alg-framework}
	\begin{algorithmic}
		\STATE {\bfseries Input:} $X \in \mathbb{R}^{N \times d}, Y \in \mathbb{R}^{N}, \{(S_iX,S_iY)\}_{i = 1}^{K}$,
		\STATE $T, T^{\dagger}\leftarrow\sum\limits_{i=1}\limits^{K} a_i, t \leftarrow 0$
        \STATE \#\#\# {\textsc{1st stage}} \#\#\#
		\FOR{$i\leftarrow1$ {\bfseries to} $K$}
		  \FOR{$j\leftarrow1$ {\bfseries to} $a_i$}
		      \STATE $\beta_{t+1} \leftarrow \text{LS\_Solver}(\beta_t, S_i X, S_i Y)$
		      \STATE $t\leftarrow t+1$
		  \ENDFOR
		\ENDFOR
        \STATE \#\#\# {\textsc{2nd stage}} \#\#\#
		\FOR{$t\leftarrow T^{\dagger}$ {\bfseries to} $T$}
		  \STATE $\beta_{t+1} \leftarrow \text{LS\_Solver}(\beta_t, X, Y)$
		\ENDFOR
		\STATE Return $\beta_{T}$
	\end{algorithmic}
\end{algorithm}

Now, we consider the efficient construction of the \(K\) sketched LS subproblems. The error \( \delta_i \) should gradually approach the precision of the OLS estimator, which means that the sketch size should increase according to a specific pattern. 
We adopt a similar idea of obtaining the sequence of the sketched data in \citep{wang2022iterative}. Specifically, the data $(S_iX, S_iY)$ can be easily extracted from the data $(S_{i+1}X, S_{i+1}Y)$, and the sketch sizes satisfy \(m_2/m_1 = \cdots = m_K/m_{K-1}\).

Moreover, different LS solvers may adopt different stopping criteria when solving the LS subproblems. In Section \ref{sec:effcntSLSE}, we will introduce an efficient implementation of the SLSE-FRS framework based on the M-IHS algorithm, providing stopping criteria and convergence analysis. This can serve as a reference for establishing relevant stopping criteria and convergence analysis when adopting other LS solvers for the LS subproblems.

\section{An Efficient SLSE-FRS}\label{sec:effcntSLSE}
After introducing the general framework of SLSE-FRS, we now present an efficient implementation. First, we introduce the concept of subspace embedding, which is of extreme importance in \textit{sketching} methods.
\begin{definition}
	Let \(X\) be a matrix of size \(N \times d\), and let \(U\) be consisted of orthonormal columns which form the basis of the column subspace of \(X\). Then, for \(\epsilon>0\), a matrix \(S\) of size \(m \times N\) is a \((1+\epsilon)\) subspace embedding for \(X\) if \(
	\left\| U^\top S^\top S U - I_d \right\| \leq \epsilon.
	\)
\end{definition}

Among the three sketching matrices introduced in Section \ref{relatedworks}, theoretically, we tend to prefer the Gaussian and SRHT sketching matrices. Although the CountSketch matrix is fast to apply, it requires sketch size \(m = \Omega(d^2)\) to guarantee a valid subspace embedding, whereas Gaussian or SRHT only needs \(m = O(d\log d)\). Additionally, for a given success probability of \(1 - \delta\), CountSketch requires a larger sketch size compared to Gaussian and SRHT. In other words, for a specified sketch size \(m\), CountSketch exhibits a higher failure probability than the other two sketching matrices.

By considering the trade-off between sketching time and theoretical guarantee, we choose to implement SLSE-FRS with SRHT. Here we construct the SRHT matrix in a similar fashion as \citep{dobriban2019asymptotics}. For an integer \(N = 2^p\) with \(p \geq 1\), We define the \(N \times N\) SRHT matrix \(S = (N/m)^{1/2}BHDP\), where \(B\) is the \(N \times N\) diagonal random sampling matrix with i.i.d. Bernoulli random variables with success probability \(m/N\), \(H\) is a normalized Hadamard matrix,
\(D \in \mathbb{R}^{N \times N}\) is a diagonal matrix with i.i.d. Rademacher random variable, and \(P \in \mathbb{R}^{N \times N}\) is a uniformly distributed permutation matrix. Lastly, we retain the non-zero rows of the matrix \(S\), which forms the SRHT matrix, and still denote it as \(S\). The following theorem implies the embedding property of the SRHT sketching matrix.
\begin{theorem}
	\citep{wang2022iterative}
	\label{thm:bigtheorem}
	Let \( S \) be a SRHT sketching matrix of size \( m \times N \) . Let \( N \) be a power of 2, \( U \in \mathbb{R}^{N \times d} \) be a column orthonormal matrix, and \( \epsilon, \delta \in (0,1) \). If
	\[
	m \geq c\epsilon^{-2}\left(d + \log\left(\frac{N}{\delta}\right)\right)\log\left(\frac{ed}{\delta}\right),
	\]
	where \( c > 0 \) is a constant, it holds that
	\[
	\Pr\left\{\left\| U^\top S^\top S U - I_d \right\| > \epsilon\right\} \leq \delta.
	\]
\end{theorem}

For the selection of the LS solver in the two stages of SLSE-FRS, we prefer to adopt the M-IHS algorithm due to its numerous beneficial features \citep{ozaslan2019iterative}. Moreover, the Hessian sketch applied in M-IHS is defined as \(\hat{H}\triangleq(\hat{S}X)^{T}\hat{S}X\), where \(\hat{S} \in \mathbb{R}^{r \times N}\) is a fixed sketching matrix.

In SLSE-FRS, the M-IHS iteration in the 1st stage can be expressed as
\begin{equation}\label{98}
	\beta_{t + 1} = \beta_t - \mu \hat{H}^{-1}(S_{i}X)^{\top}(S_{i}X\beta_t - S_{i}Y) + \eta ({\beta _t} - {\beta _{t - 1}}),
\end{equation}
where \(S_i\) is the \(i\)-th sketching matrix in the \(i\)-th sketched LS subproblem.
The M-IHS iteration in the 2nd stage can be expressed as
\begin{equation}\label{99}
	\beta_{t + 1} = \beta_t - \mu \hat{H}^{-1} X^{\top}(X\beta_t - Y) + \eta ({\beta _t} - {\beta _{t - 1}}).
\end{equation}

In fact, the formulae (\ref{98}) and (\ref{99}) can be regarded as using a number of preconditioned Richardson extrapolation iterations, namely $a_i$ iterations for (\ref{98}) and $T-T^{\dagger}$ iterations for (\ref{99}), with the preconditioner \(\hat{H}\) to solve the normal equations of the sketched LS subproblems and the full-scale original LS problem, namely
\begin{equation}\label{linearsystem}
	(S_{i}X)^{\top}S_{i}X\beta = (S_{i}X)^{\top}S_{i}Y
\end{equation}
and
\begin{equation}\label{fulllinearsystem}
	X^{\top}X\beta = X^{\top}Y.
\end{equation}

The floating-point operation (FLOPs) cost of the formulae (\ref{98}) and (\ref{99}) can be obtained based on the results in \citep{golub2013matrix}. Given \( \beta_t\), \(\hat{H}\), and the sketched data \((S_{i}X, S_{i}Y), i = 1, \ldots, K\), one sketched M-IHS iteration using formula (\ref{98}) in the 1st stage requires \(\{(4d+1)m_i+2d^2+5d\}\) FLOPs, while one full-scale M-IHS iteration with the data \((X, Y)\) requires \(\{(4d+1)N+2d^2+5d\}\) FLOPs. In one iteration, the first term of the cost is reduced from \(\{(4d+1)N\}\) to \(\{(4d+1)m_i\}\). It is noteworthy that the number of full-scale M-IHS iterations required in the 2nd stage is significantly less than the number of full-scale M-IHS iterations used throughout the entire iteration process.

Based on the embedding property of SRHT, the SRHT sketching matrix \(\hat{S} \in \mathbb{R}^{r \times N}\) adopted in the Hessian sketch \(\hat{H}\triangleq(\hat{S}X)^{T}\hat{S}X\) should be of sketch size \(r=O(d\log d)\). We follow the idea in \citep{wang2022iterative} to determine the sketch size \({m_i}\), \(i = 1,\ldots, K\), of the sketched LS subproblem (\ref{sketchedproblem}), which grows by a factor of \(2=m_{i+1}/m_i\), \(i = 1,\ldots, K-1\). In this way, we construct \(K = \log_2(N/m_1)\) sketched LS subproblems. According to the aforementioned iterations for (\ref{linearsystem}) and (\ref{fulllinearsystem}), the Hessian sketch \(\hat{H}\) serves as a randomized preconditioner (or approximation) of both \((S_{i}X)^{\top}S_{i}X\) and \(X^{\top}X\). Therefore, the sketch size of the smallest sketched LS subproblem should be larger than the sketch size of \(\hat{H}\), i.e., \(m_1 > r\).

Now, we consider the details of the efficient construction of the sketching matrix \(S_{i} \in \mathbb{R}^{{m}_{i} \times N}\), \(i = 1, \ldots, K\). Firstly, we left-multiply the original data \((X, Y)\) by \(HDP\) to get the full-scale sketched data \((S_0X, S_0Y):= (HDPX, HDPY)\). Secondly, we can easily construct the sketched data sequence \((S_iX, S_iY)\), \(i = 1,\ldots, K\), by sampling the rows of the data \((S_0X, S_0Y)\), namely by left-multiplying the random matrix \(B_i \in \mathbb{R}^{m_i \times N}\), which is consisted of the \(m_i\) non-zero rows of \(B\). The \(i\)-th sketched data can be represented as
\begin{equation*}
	(S_iX,S_iY) = \sqrt{\frac{N}{m_{i}}}B_i(S_0X,S_0Y).
\end{equation*}
The bottleneck of constructing the \(K\) sketched LS subproblems is to apply the Hadamard transform to a matrix of size \(N\times d\), which costs \(Nd\log_2 N\) FLOPS. Hence, the overall sketching time complexity is \(Nd\log_2 N\).

In summary, the above discussion leads to an efficient implementation of SLSE-FRS, as described in Algorithm \ref{Alg-SRHT}.

\begin{algorithm}[bt]
	\caption{An Efficient SLSE-FRS}
	\label{Alg-SRHT}
	\begin{algorithmic}
		\STATE {\bfseries Input:} $X \in \mathbb{R}^{N \times d}, Y \in \mathbb{R}^{N}, \hat{H}, \mu, \eta, (S_0X,S_0Y)$,
		\STATE $K, {\{B_i\}}_{i = 1}^{K}, {\{a_i\}}_{i = 1}^{K}, {\{m_i\}}_{i = 1}^{K}, T, T^{\dagger}\leftarrow\sum\limits_{i=1}\limits^{K} a_i, t \leftarrow 0$
        \STATE \#\#\# {\textsc{1st stage}} \#\#\#
		\FOR{$i\leftarrow1$ {\bfseries to} $K$}
		\STATE $(S_iX,S_iY)\leftarrow \sqrt{\frac{N}{m_{i}}}B_i(S_0X,S_0Y)$
		\FOR{$j\leftarrow1$ {\bfseries to} $a_i$}
		\STATE $\beta_{t+1} \leftarrow \beta_t - \mu \hat{H}^{-1} \nabla f(\beta_t; S_i X, S_i Y) + \eta (\beta_t - \beta_{t-1})$
		\STATE $t\leftarrow t+1$
		\ENDFOR
		\ENDFOR
        \STATE \#\#\# {\textsc{2nd stage}} \#\#\#
		\FOR{$t\leftarrow$ $T^{\dagger}$ {\bfseries to} $T$}
		\STATE $\beta_{t+1} \leftarrow \beta_t - \mu \hat{H}^{-1} \nabla f(\beta_t; X, Y) + \eta (\beta_t - \beta_{t-1})$
		\ENDFOR
		\STATE Return $\beta_{T}$
	\end{algorithmic}
\end{algorithm}

\section{The convergence of SLSE-FRS}\label{sec:theorem}

This section offers the theoretical assurance of SLSE-FRS.
Let ${\beta}_{j}^{i}$ represent the $j$-th iterate of the $i$-th sketched LS subproblem. The \(K\) iterates \({{\beta}_{a_i}^{i}}\), \(i = 1,\ldots, K\), of the sketched LS subproblems can serve as a sequence of estimators for the true parameter vector \(\beta\) with increasing precision \(\delta_{i}\). For any \(\beta_0 \in \mathbb{R}^{d}\), we set \(\beta_{0}^1 = \beta_0\) and \(\beta_{0}^{i} = {\beta_{a_{i-1}}^{i-1}}, i = 2,\ldots, K\). The following theorem provides the convergence behavior of the iterations for the \(i\)-th sketched LS subproblem.

\begin{theorem}
	\label{thm:convergence_theorem1}
	In the \(i\)-th sketched LS subproblem of SLSE-FRS in Algorithm \ref{Alg-SRHT}, let \(N, m_1\) be powers of 2, \(\delta \in (0, 1)\), \(\epsilon \in (0, 1/10)\), \(|\mu-1| \leq 1/4\), and \(\eta = 53/36 - \sqrt{17}/3\). Let
	\[
	\begin{aligned}
		& r \geq c\epsilon^{-2}\left[d + \log\left(\frac{N}{\delta}\right)\right]\log\left(\frac{ed}{\delta}\right), \quad m_1 \textgreater r,\\
	\end{aligned}
	\]
	where \( c > 0 \) is a constant. There exists a constant \(M_i \geq 1\), such that if \(a_i > M_i\), with probability at least \(1-2\delta\), we have
	\begin{align*}
		\mathbb{E} \| X(\beta_{a_i}^{i} - {\beta}) \|
		\leq
		\left(\frac{1}{3}\right)^{a_i}
		&\mathbb{E}\|
		X(\beta_{0}^{i} - \beta)
		\|
	    +
		\left[
		1+\left(\frac{1}{3}\right)^{a_i}
		\right]\delta_i.
	\end{align*}
\end{theorem}
According to Theorem \ref{thm:convergence_theorem1}, the prediction error decays exponentially. For sufficiently large \(a_i\), the estimator \(\beta_{a_i}^{i}\) of the \(i\)-th sketched LS subproblem can achieve the precision \(\delta_i\). Based on Theorem \ref{thm:convergence_theorem1}, we provide the convergence assurance of SLSE-FRS in the following theorem.


\begin{theorem}\label{Thf}
	In Algorithm \ref{Alg-SRHT}, let \(N,m_1\) be powers of 2, \(T^\dagger=\sum_{i=1}^K a_i\),  \(\delta \in (0, 1)\), \(\epsilon \in (0, 1/10)\), \(|\mu-1| \leq 1/4\), and \(\eta = 53/36 - \sqrt{17}/3\). Let
	\[
	\begin{aligned}
		& r \geq c\epsilon^{-2}\left[d + \log\left(\frac{N}{\delta}\right)\right]\log\left(\frac{ed}{\delta}\right), \quad m_1 \textgreater r,\\
	\end{aligned}
	\]
	where \( c > 0 \) is a constant. Let \(s=T-T^\dagger\), and define
	\[
	\mathcal{T}_i:=\sum_{\ell=1}^i a_\ell,
	\qquad i=1,\ldots,K,
	\]
	with \(\mathcal{T}_0=0\) and \(\mathcal{T}_K=T^\dagger\). Define \(\delta_i:=\mathbb{E}\|X(\tilde{\beta}_i-\beta)\|,\) \(\delta_\star:=\mathbb{E}\|X(\hat{\beta}-\beta)\|.\)
	Then there exists a constant \({M} \geq 1\), such that if \(a_i > {M}\) for all \(i = 1,\ldots, K\) and $s > M$, with probability at least \(1-2\delta\), for any \(\beta_0\in\mathbb{R}^d\), the iterate \(\beta_T\) satisfies
	\begin{equation}\label{basic_convergence}
		\mathbb{E}\|X(\beta_T-\beta)\|
		\le  \left(\frac13\right)^T\mathbb{E}\|X(\beta_0-\beta)\|+\sum_{i=1}^{K}\left[1+\left(\frac13\right)^{a_i}\right]\left(\frac13\right)^{T-\mathcal{T}_i}\delta_i
		+\left[1+\left(\frac13\right)^s\right]\delta_\star .
	\end{equation}
	Furthermore, let \(C>0\) be a constant such that \(\delta_i\le C\delta_\star, i=1,\ldots,K\), then
	\begin{equation}\label{final_precision}
		\mathbb{E}\|X(\beta_T-\beta)\|
		\le
		\left(\frac13\right)^T
		\mathbb{E}\|X(\beta_0-\beta)\|
		+
		\left[
		1+
		\left(\frac13\right)^s(1+2C)
		\right]\delta_\star .
	\end{equation}
\end{theorem}

Under the parameter settings specified in Theorems \ref{thm:convergence_theorem1} and \ref{Thf}, the convergence rate of SLSE-FRS is bounded above by $1/3$. We note that, according to the idea of the Proof of Theorem \ref{thm:convergence_theorem1} in Appendix \ref{B}, there is a possibility to obtain an even sharper upper bound for the convergence rate of SLSE-FRS by carefully selecting alternative parameter settings.

In addition, the relative error condition \(\delta_i\le C\delta_\star\) is only used to derive the simplified OLS-level precision estimate from the basic convergence bound \eqref{basic_convergence}. This condition is natural in the present sequential framework, since the sketched LS subproblems are constructed with increasing sketch sizes and are expected to provide progressively refined estimators. For a fixed finite sequence of sketched subproblems, when \(\delta_\star>0\), one may take
\begin{equation*}
	C=\max_{1\le i\le K}\frac{\delta_i}{\delta_\star}.
\end{equation*}
Notably, the second term on the right-hand side of \((\ref{final_precision})\) is governed by the factor \(1+\left(1/3\right)^s(1+2C)\). For fixed \(C\), this factor tends to \(1\) as the number \(s=T-T^\dagger\) of full-scale refinement steps increases. Therefore, as \(T\) increases with \(s=T-T^\dagger\to\infty\), the first term converges to zero, while the second term converges to \(\delta_\star\). Consequently, the final error is essentially at the noise level, which means that SLSE-FRS reaches the OLS-level prediction precision.

In practical implementations, the determination of the iteration count \(a_i\) in the \(i\)-th sketched LS subproblem is essential. Here, we provide an estimation of \(a_i\) required by the \(i\)-th sketched LS subproblem to achieve the precision \(\delta_i\). In the \(i\)-th sketched LS subproblem, we expect that SLSE-FRS returns an iterate \(\beta_{a_i}^{i}\) satisfying the condition (\ref{11}). With a prescribed tolerance \(\omega \in (0,1)\), the count \(a_i\) can be determined by requiring
\begin{align}\label{14}
	\mathbb{E}\|X({\beta}_{a_i}^i - {\beta})\|
	& \leq
	\mathbb{E}\|X({\beta}_{a_i}^i - \tilde{\beta}^i)\| + \mathbb{E}\|X({\tilde{\beta}^i} - \beta)\| \nonumber \\
	& \leq (1+\omega)\mathbb{E}\|X({\tilde{\beta}^i} - \beta)\|.
\end{align}
The following theorem provides an explicit lower bound of \({a_i}\), \(i = 2,\ldots, K\). Related discussion and complete proof details are presented in Appendix~\ref{appendixa5}. The lower bound of \(a_1\) is determined by Theorem \ref{a1} in Appendix \ref{appendixa1}.

\begin{theorem}\label{lowerBoundai}
	In Algorithm \ref{Alg-SRHT}, for a prescribed tolerance \(\omega \in (0,1)\), the iteration count \(a_i\) needed for the \(i\)-th sketched LS subproblem to fulfill (\ref{14}) satisfies
	\[
	a_i
	\geq
	\log_3 \left[ \frac{(1+\omega)r(i-1, i) + 1}{\omega} \right]
	> 0,
	\]
	where \( r(i-1,i) = \sqrt{(m_{i} - d) / (m_{i-1} - d)} \) .
\end{theorem}
Although the iteration count, e.g., $a_i$ and $T$, should be a positive integer, for convenience of explanation, we will directly use positive real numbers to represent the iteration count in the following discussion.
If we let the iteration count \(a_i\) equals to the above lower bounds, then
Theorem \ref{lowerBoundai} implies that the sequence $\{a_i\}_{i=2}^K$ are monotonically increasing with respect to $i$, namely \(a_2<\ldots<a_K\) (refer to Appendix \ref{cmpxtyAlgSLSEPrf}), which can serve as a guidance for determining \(a_i\) in Algorithm \ref{Alg-SRHT}.

To end this section, we summarize the computational complexity of Algorithm \ref{Alg-SRHT} in the asymptotic sense in the theorem as follows. Here, the complexity is only considered for its main part, i.e., the dominant portion.

\begin{theorem} \label{cmpxtyAlgSLSE}
  Under the same conditions as Theorem \ref{lowerBoundai}, let Algorithm \ref{Alg-SRHT} terminate when it finds an estimator that achieves the noise level \(\sigma\). Let the positive integer $K=O(1)$, and the sketch sizes satisfy \(m_{i+1}/m_i=2\), $i=1,\ldots,K-1$, and $m_K=N/2$. For a specified $\omega\in(0,1)$, let $a_i$ ($i=1,\ldots,K$) take exactly their lower bounds. In the sense of asymptotic meaning, that is, let $d/N\rightarrow\gamma\in(0,2^{-K})$ as $N\rightarrow+\infty$, it holds that $\log_3(1/\omega)<a_i<\alpha$ ($i=1,\ldots,K$) with $\alpha=\log_3\{[1+(1+\omega)/\sqrt{2}]/\omega\}$ if $1 < \mathbb{E}\|X({\beta}_0 - \tilde{\beta}_1)\| / \mathbb{E}\|X(\tilde{\beta}_1 - \beta)\| < 1+(1+\omega)/\sqrt{2}$. Hence, the costs (dominant portion only) of all stages of Algorithm \ref{Alg-SRHT} are listed below:
  \begin{itemize}
    \item The initialization stage: $Nd\log_2N$;
    \item The 1st stage: $4\alpha Nd$;
    \item The 2nd stage: $4[\log_3(\omega^{K}/\sigma)]Nd$ with $\omega>\sigma^{1/K}$.
  \end{itemize}
\end{theorem}

According to Theorem \ref{cmpxtyAlgSLSE}, the complexity of Algorithm \ref{Alg-SRHT} is dominated by the initialization stage, namely $Nd\log_2N$. On one hand, in scenarios where $\log_2N$ can be considered as $O(1)$, the total complexity of Algorithm \ref{Alg-SRHT} is $O(Nd)$. On the other hand, if the SRHT sketching matrix is replaced by the CountSketch sketching matrix, the complexity at the initialization stage of Algorithm \ref{Alg-SRHT} can be reduced to $O(Nd)$. If theorems (currently non-existent) similar to Theorems \ref{Thf}-\ref{cmpxtyAlgSLSE} can be proved, the theoretical complexity of Algorithm \ref{Alg-SRHT} with CountSketch becomes $O(Nd)$. The experiments in the next section show that Algorithm \ref{Alg-SRHT} with CountSketch is faster than Algorithm \ref{Alg-SRHT} with SRHT in the sense of computing time.

We need to emphasize that SLSE-FRS is the first algorithmic framework to integrate \textit{Sketch-and-Solve} and \textit{Iterative-Sketching}. Specifically, Algorithm \ref{Alg-SRHT} employs an integration of the SRHT-based \textit{Sketch-and-Solve} method with M-IHS, and the results of this section provide theoretical support for the aforementioned integration scheme. As a new algorithmic framework, SLSE-FRS can flexibly adapt to other potential integration schemes, and the theoretical framework provided in this section can offer possible paths for the theoretical analysis of any new algorithms within the SLSE-FRS framework.

\section{Numerical experiments} \label{sec:experiment}
In this section, we present some numerical experiments to show the effectiveness and efficiency of SLSE-FRS. We compare its performance with IDS in \citep{wang2022iterative}, PCG in \citep{lacotte2021faster}, and M-IHS in \citep{ozaslan2019iterative}. Our experiments follow standard configurations in prior works \citep{pilanci2016iterative,pilanci2017newton,wang2022iterative}. See environment details in Appendix \ref{details}.

We begin with generating linear models \(Y = X \beta + \zeta\). The feature data matrix $X$ is generated as follows: First, a Gaussian random matrix $G$ of size $N \times d$ is generated. Then, its columns are scaled, i.e., we set \(X = G D,\) where $D = \mathrm{diag}(\tilde{\sigma}_1, \ldots, \tilde{\sigma}_d)$ is a diagonal matrix with predetermined elements, thereby controlling the condition number of $X$. The components of the true parameter vector $\beta$ are independently sampled from the standard Gaussian distribution, and the components of the noise vector $\zeta$ are independently sampled from $\mathcal{N}(0,10^{-8})$. Finally, the response vector $Y$ is generated via the model (\ref{linear_model}). Based on these models, we construct LS test problems. The estimators of the true parameter vector \(\beta\) are computed by SLSE-FRS, IDS, PCG, and M-IHS.

In our experiments, we use small noise levels and large condition numbers to ensure that the algorithms require many iterations to reach the target accuracy, thereby clearly demonstrating  their performance under challenging scenarios. Under moderate noise levels or well-conditioned data matrices, all tested methods would converge within few iterations, making it difficult to meaningfully distinguish their convergence behavior and computational performance.

To evaluate the precision of each iterate \(\beta_t\) for \(t = 1, \ldots, T\), we consider the prediction error \(\Delta_t := \|X(\beta_t - \beta)\|^2\). According to \citep{pilanci2016iterative}, the iterates are expected to achieve the LS error, defined as \(\Delta:=\)\(\|X( \hat{\beta}- \beta)\|^{2}\), which is approximately equal to \(\sigma^2 d \).

We set \(T = 100\) to ensure all algorithms converge and achieve the OLS precision (as small as the LS error \(\Delta\)) within this limit. In the following experiments, we test different sample sizes $N$, feature sizes \(d\), condition numbers $\kappa$ of the feature matrix $X$, and the Hessian sketch size is set to \(r = 6d\) for all test models. IDS employs the iteration parameter \(\mu = \frac{(1 - d/r)^2}{1 + d/r}\) as suggested in \citep{wang2022iterative}. SLSE-FRS and M-IHS adopt the iteration parameters \(\mu = (1 - \eta)^2 \) and \(\eta = d/r\) following \citep{ozaslan2019iterative}. SLSE-FRS takes the expansion ratio $m_{i+1}/m_i=2$ with $m_1=8d$ and $m_K=N/2$ for sketch sizes.
PCG uses the SRHT sketching matrix to construct the preconditioner like \citep{lacotte2021faster}. The results of the subsequent report are the average of 10 independent runs.

\subsection{Determination of $a_i$}

In this subsection, we aim to determine feasible values of $a_i$ for all subsequent experiments at a sufficiently low cost to ensure the efficiency of SLSE-FRS.

According to Theorem \ref{lowerBoundai}, let $a_i$ be its lower bound. We test SLSE-FRS on an LS problem with \(N = 2^{20}\), \(d = 2^{6}\), and \(\kappa = 10^4\). When $\omega$ takes values from the set $\{2^{-4}$, $2^{-3}$, $2^{-2}$, $2^{-1}$, $1\}$, we obtain the curve on the left plot of Figure \ref{fig-determine-ai}, where the $\omega$-axis uses a logarithmic scale. Since $a_i$ is monotonically increasing with respect to $i=2,\ldots,K$, we only present the curves for $a_2$ and $a_K$, with other curves necessarily lying between them. The two curves almost coincide, with the curve of $a_K$ slightly higher than that of $a_2$, and both curves decrease linearly as $\log(\omega)$ increases. As $\omega$ characterizes the solution precision of all sketched LS subproblems, the curves in this plot indicate that as the solution precision improves, the value of $a_i$ increases. Among all tested values of $\omega$, the maximum value of $a_i$ is approximately 3. This suggests that the value of $a_i$ does not need to be too large, if high-precision solutions (related to small $\omega$) for sketched LS subproblems are not required.

Based on the aforementioned facts, we conducted further experiments. Specifically, we set $a_i = \mathrm{NoI}$ (the number of iterations of the sketched LS subproblems) for all $i$. The right plot in Figure \ref{fig-determine-ai} shows the curve of the computing time of SLSE-FRS versus the value of $\mathrm{NoI}$. This curve indicates that the total computing time of SLSE-FRS achieves its minimum at $\mathrm{NoI}=2$, while its total computing time at $\mathrm{NoI}=3$ is very close to that at $\mathrm{NoI}=2$. Based on these observations, we recommend using $a_i = 2$ or 3 for all $i$ in subsequent experiments.

\begin{figure}[htbp]
	\centering
	\subfigure[$a_2$, $a_K$ vs. $\omega$.]{\includegraphics[width=3in]{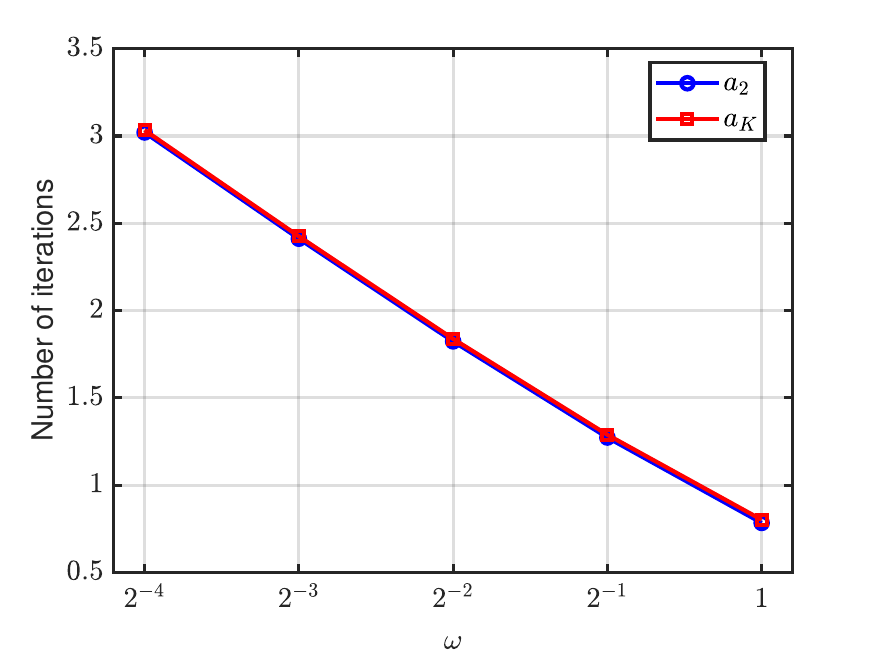}}
	\subfigure[Computing time vs. $a_i=\mathrm{NoI}$ (number of iterations).]{\includegraphics[width=3in]{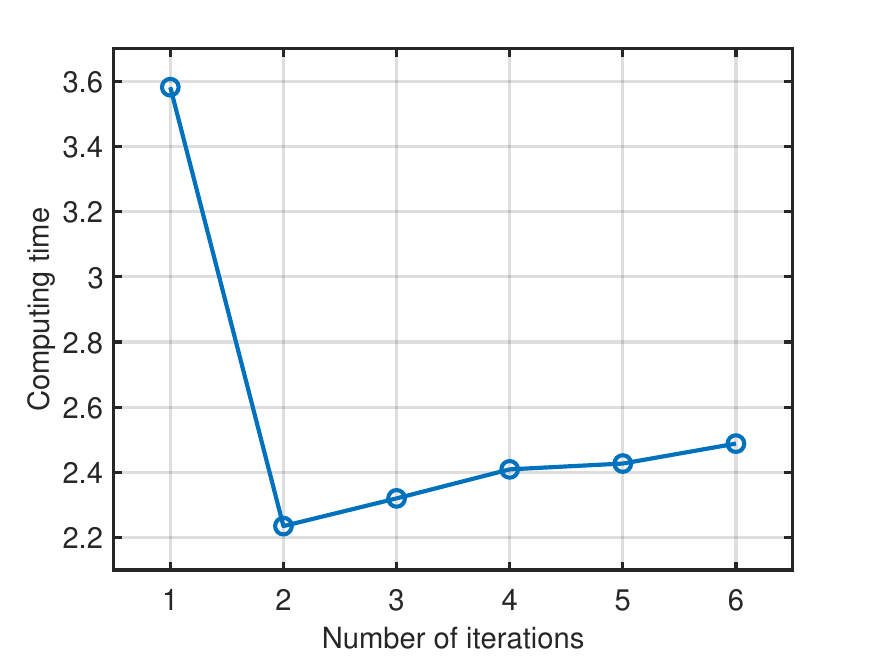}}
	\caption{The iteration control parameter $a_i$ for SLSE-FRS: $a_2$, $a_K$ versus $\omega$; computing time versus $a_i=\mathrm{NoI}$ (the number of iterations of the sketched LS subproblems) for all $i$.}
	\label{fig-determine-ai}
	\vskip -0.1in
\end{figure}

\subsection{Precision and efficiency}
This part is designed to verify two key things: the first is that the output of SLSE-FRS can achieve the LS error \(\Delta\), and the second is that SLSE-FRS improves the state-of-the-art computing time for high-precision LS estimators.

We test algorithms on LS problems with \(N \in \{2^{17}, 2^{18}, 2^{19}, 2^{20}\}\), $d=2^6$, and \( \kappa=10^{4}\). We take \(a_i = 2\) in SLSE-FRS. In Figure \ref{fig1}, the left plot illustrates that the prediction error \(\Delta_T \) of SLSE-FRS achieves the same error-level as the LS error \(\Delta \) of the OLS estimator (approximately $\sigma^2 d = 6.4\times10^{-7}$ with the noise variance $\sigma^2=10^{-8}$). This plot confirms that SLSE-FRS has achieved the OLS precision.

Additionally, we construct a special LS problem with \(N = 2^{20}\) and \(d = 2\) to show the convergence paths of SLSE-FRS and IDS. Specifically, we draw 100 convergence paths per method in the two-dimensional plane based on their independent runs. According to the right plot of Figure \ref{fig1}, it can be clearly observed that, compared to the iteration paths of IDS, those of SLSE-FRS are more concentrated. This phenomenon arises because SLSE-FRS can better utilize the sketched LS estimators as control points for the iteration paths, thereby determining a more accurate and stable search direction.

\begin{figure}[htbp]
	\centering
	\subfigure[LS error \(\Delta\) and prediction error \(\Delta_T \) of SLSE-FRS.]{\includegraphics[width=3in]{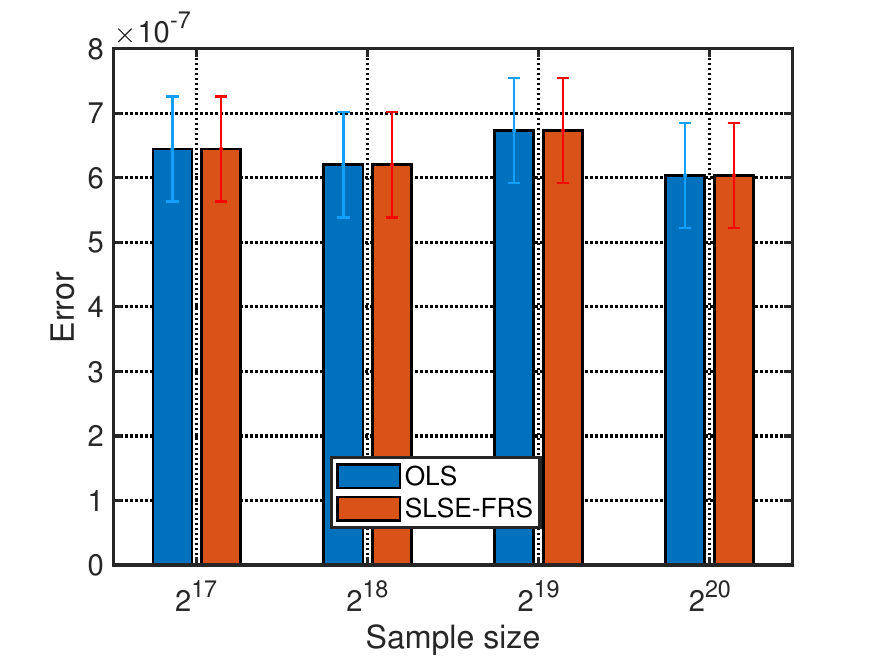}}
	\subfigure[Iteration paths of SLSE-FRS and IDS.]{\includegraphics[width=3in]{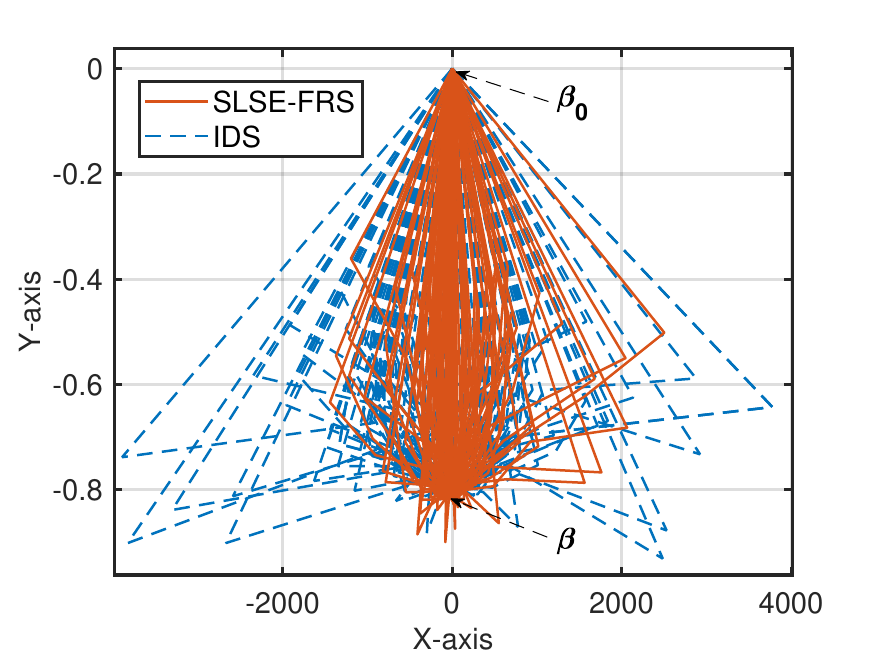}}
	\caption{\(\Delta_T \) of SLSE-FRS and \(\Delta\) of OLS and iteration paths of SLSE-FRS and IDS.}
	\label{fig1}
	\vskip -0.1in
\end{figure}

Figures \ref{fig17}--\ref{fig20} present the LS problems with $N \in \{2^{17}, 2^{18}, 2^{19}, 2^{20}\}$, $d = 2^{6}$ and $\kappa = 10^4$. These figures illustrate \(\Delta_t\) versus iteration count and actual computing time in seconds of SLSE-FRS against IDS and PCG. As shown in figures, SLSE-FRS converges much faster than PCG and IDS. Furthermore, SLSE-FRS demonstrates a significant improvement in computational efficiency. The computing time for IDS is approximately twice that of SLSE-FRS, while PCG takes roughly three times as long.

\begin{figure}[htbp]
	\centering
	\vskip 0.1in
	\subfigure[\(N = 2^{17}, d = 2^6, \kappa = 10^4\)]{\includegraphics[width=3in]{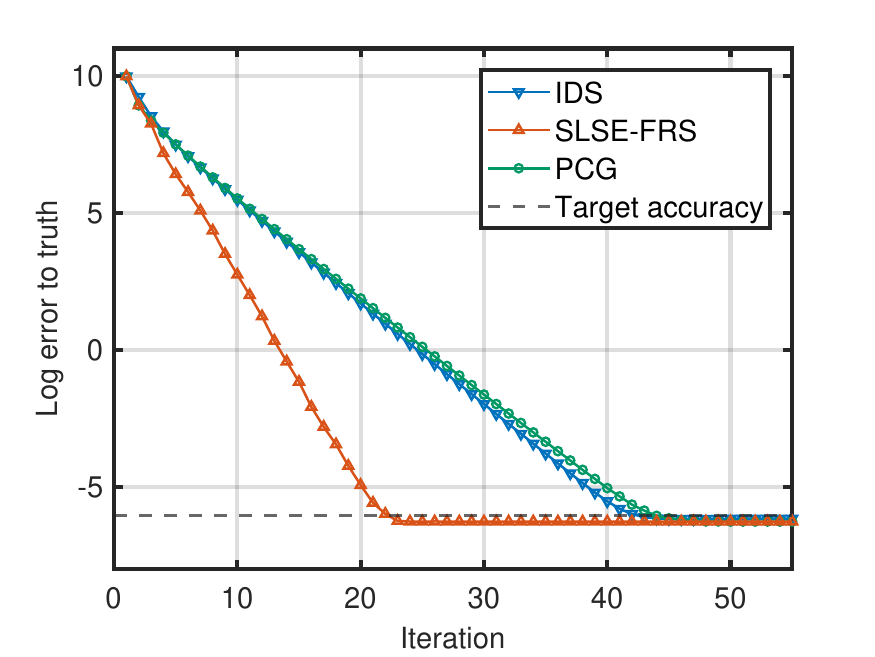}}
	\subfigure[\(N = 2^{17}, d = 2^6, \kappa = 10^4\)]{\includegraphics[width=3in]{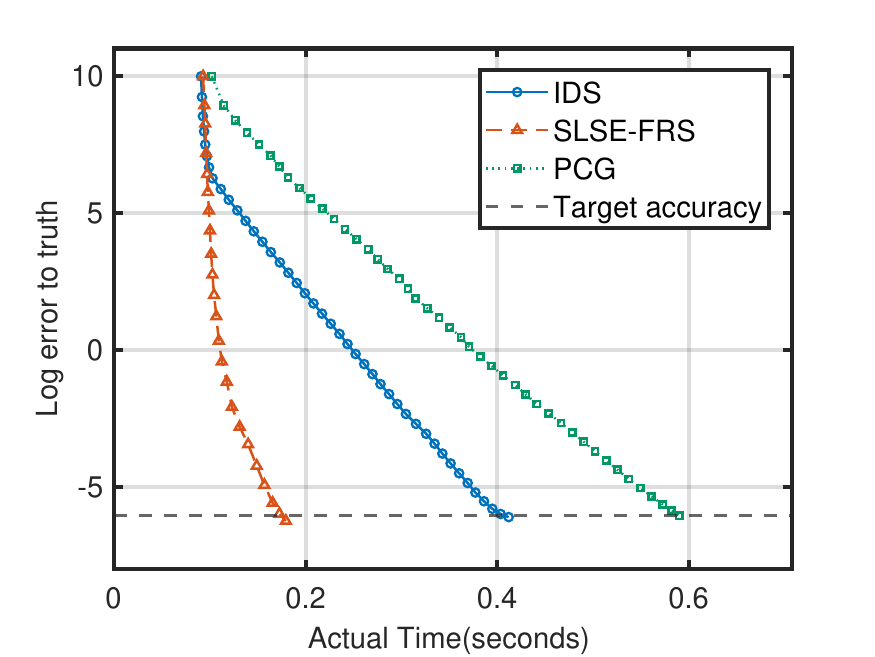}}
	\caption{\(\Delta_t \) versus iterations and actual computing time for SLSE-FRS, IDS and PCG.}
	\label{fig17}
	\vskip -0.1in
\end{figure}

\begin{figure}[htbp]
	\centering
	\vskip 0.1in
	\subfigure[\(N = 2^{18}, d = 2^6, \kappa = 10^4\)]{\includegraphics[width=3in]{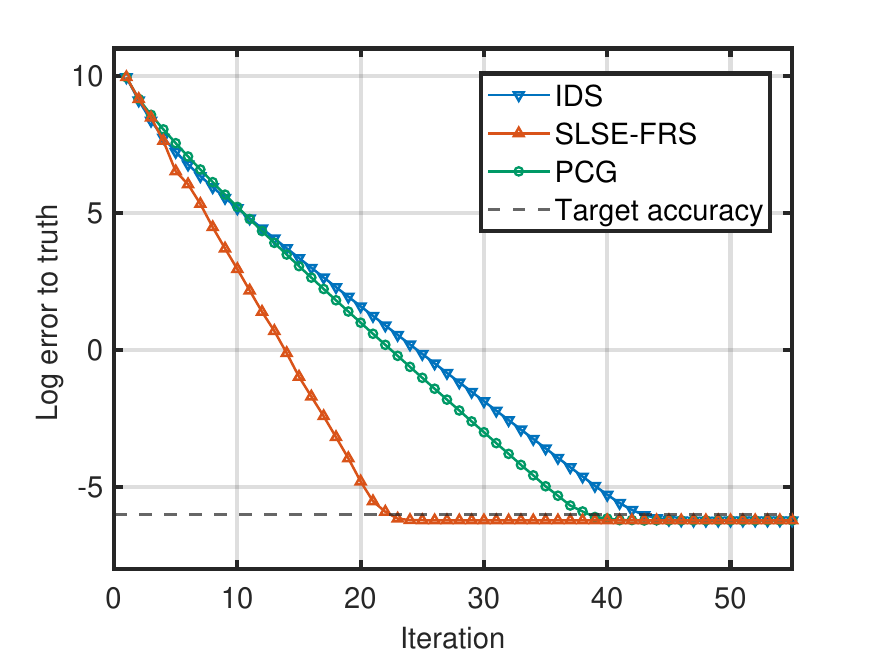}}
	\subfigure[\(N = 2^{18}, d = 2^6, \kappa = 10^4\)]{\includegraphics[width=3in]{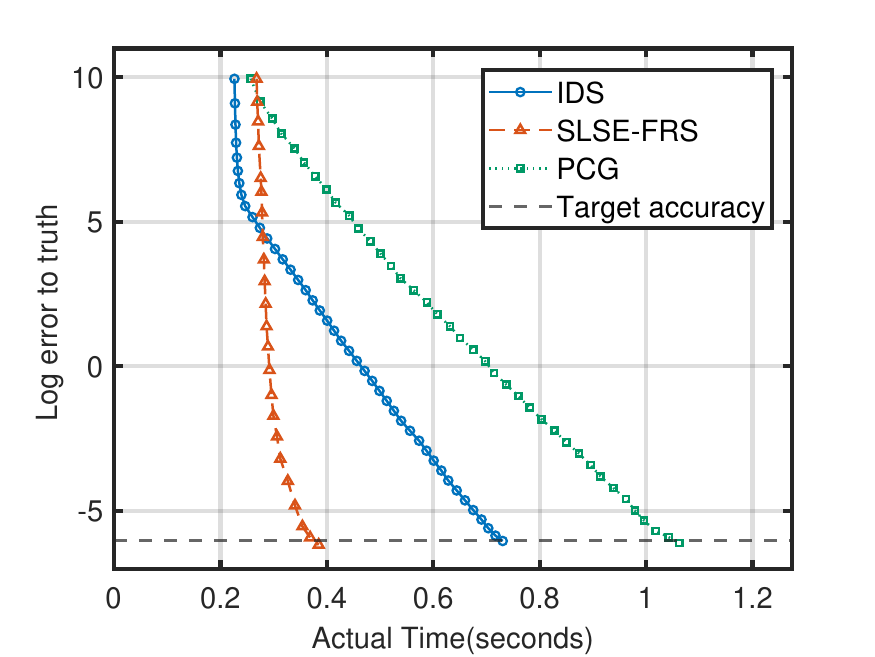}}
	\caption{\(\Delta_t \) versus iterations and actual computing time for SLSE-FRS, IDS and PCG.}
	\label{fig18}
	\vskip -0.1in
\end{figure}

\begin{figure}[htbp]
	\centering
	\vskip 0.1in
	\subfigure[\(N = 2^{19}, d = 2^6, \kappa = 10^4\)]{\includegraphics[width=3in]{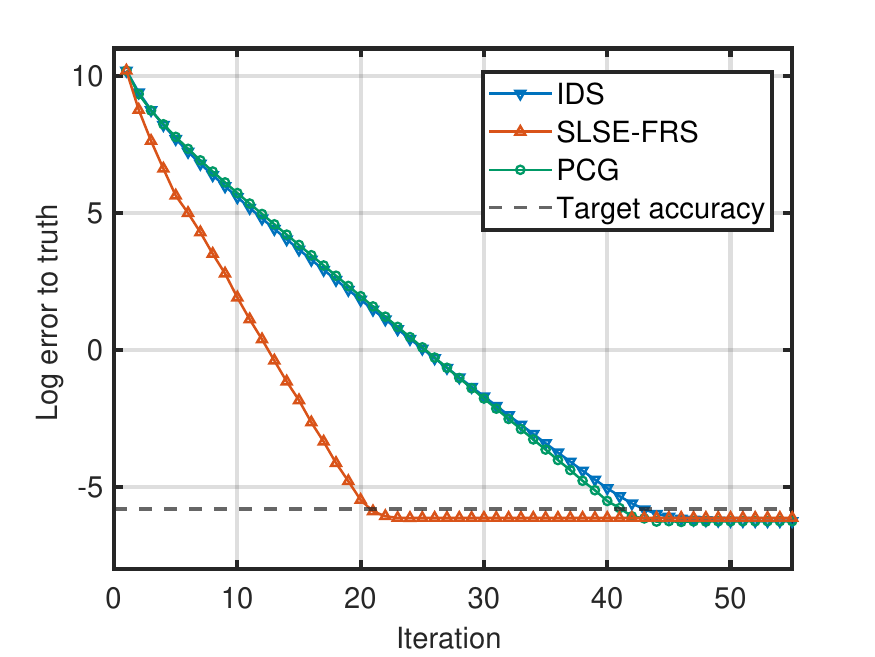}}
	\subfigure[\(N = 2^{19}, d = 2^6, \kappa = 10^4\)]{\includegraphics[width=3in]{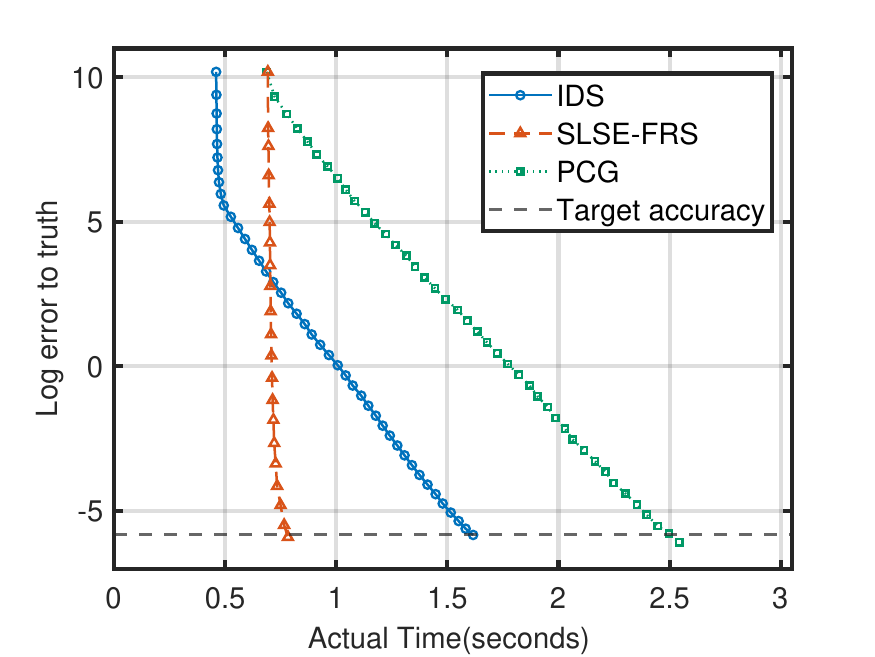}}
	\caption{\(\Delta_t \) versus iterations and actual computing time for SLSE-FRS, IDS and PCG.}
	\label{fig19}
	\vskip -0.1in
\end{figure}

\begin{figure}[htbp]
	\centering
	\vskip 0.1in
	\subfigure[\(N = 2^{20}, d = 2^6, \kappa = 10^4\)]{\includegraphics[width=3in]{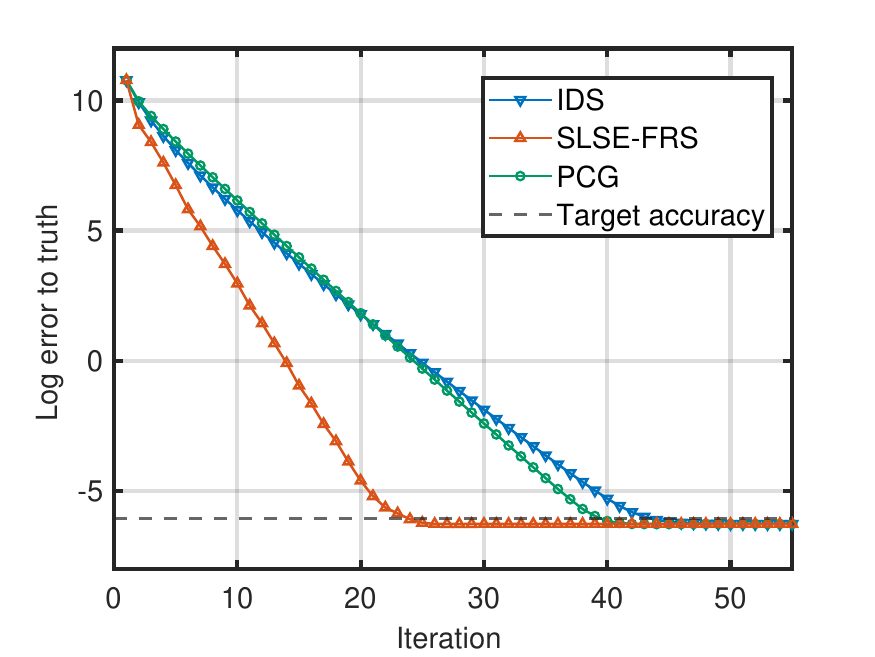}}
	\subfigure[\(N = 2^{20}, d = 2^6, \kappa = 10^4\)]{\includegraphics[width=3in]{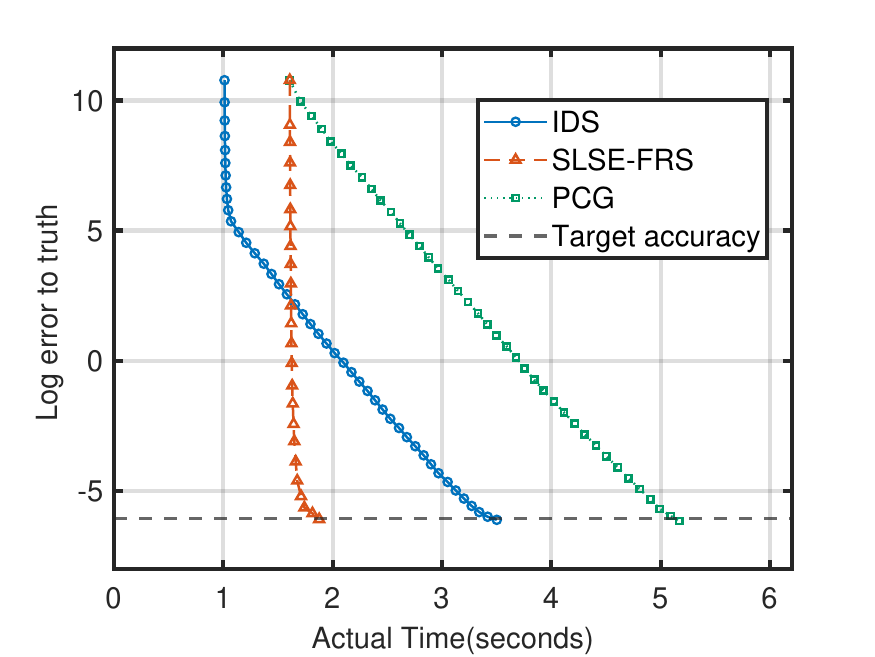}}
	\caption{\(\Delta_t \) versus iterations and actual computing time for SLSE-FRS, IDS and PCG.}
	\label{fig20}
	\vskip -0.1in
\end{figure}

We include M-IHS to further verify that SLSE-FRS achieves a convergence speed comparable to \textit{Iterative-Sketching} methods directly applied to the full-scale LS problem while significantly reducing computational costs. Under the setting of \(N = 2^{20},\ d = 2^{6},\ \kappa = 10^{8} \), Figure \ref{fig6} clearly demonstrates the rationality of the SLSE-FRS framework design. SLSE-FRS strategically utilizes an optimal quantity of data in each iteration, ensuring efficiency and precision. In contrast, M-IHS may involve an excessive amount of data in early iterations, unnecessarily increasing computational costs.

\begin{figure}[htbp]
	\centering
	\vskip 0.1in
	\subfigure[\(N = 2^{20}, d = 2^6, \kappa = 10^8\)]{\includegraphics[width=3in]{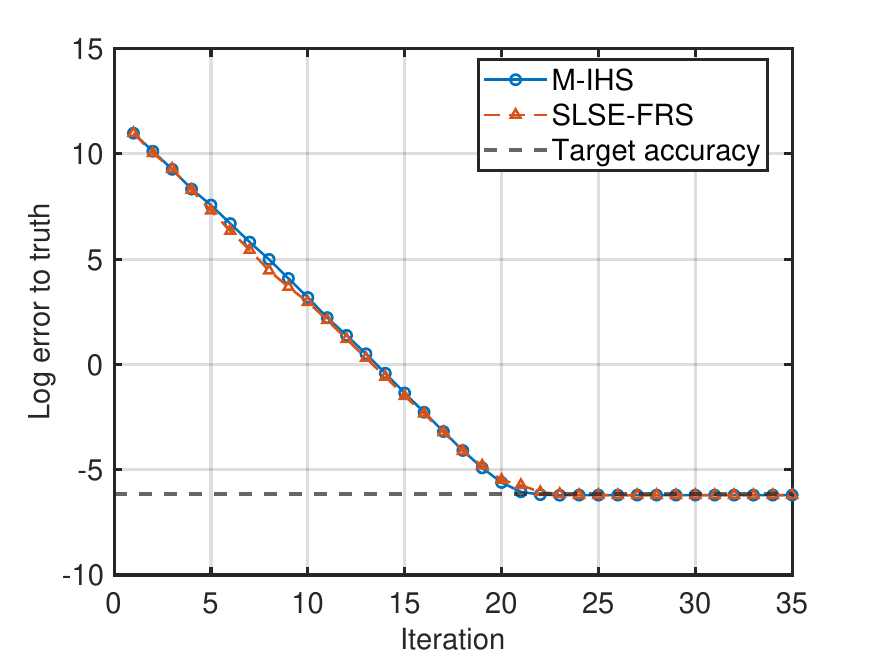}}
	\subfigure[\(N = 2^{20}, d = 2^6, \kappa = 10^8\)]{\includegraphics[width=3in]{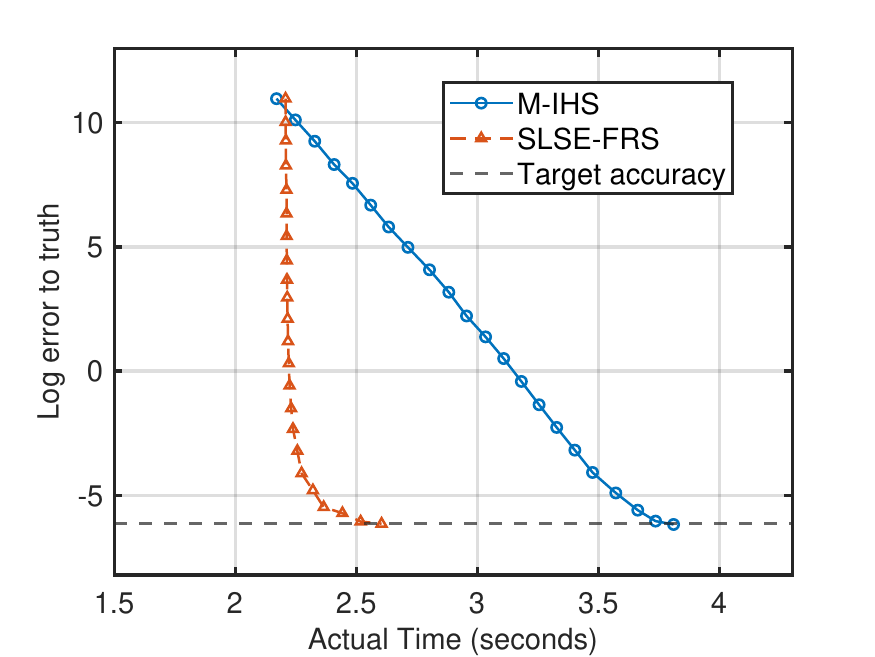}}
	\caption{\(\Delta_t \) versus iterations and actual computing time for SLSE-FRS and M-IHS.}
	\label{fig6}
	\vskip -0.1in
\end{figure}

Moreover, we provide one larger-scale experiment with \(N = 2^{20},\ d = 2^{10},\ \kappa = 10^{8} \) in Figure \ref{fig5}. We only compare IDS with SLSE-FRS, excluding PCG and M-IHS due to their much longer computing time. SLSE-FRS still maintains its superior performance in this setting.
\begin{figure}[htbp]
	\centering
	\vskip 0.1in
	\subfigure[\(N = 2^{20}, d = 2^{10}, \kappa = 10^8\)]{\includegraphics[width=3in]{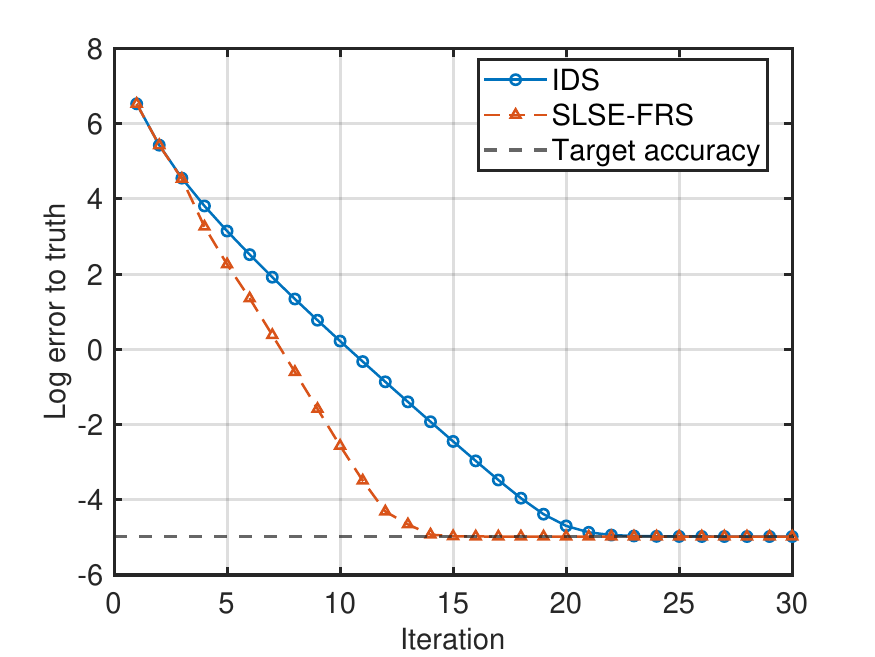}}
	\subfigure[\(N = 2^{20}, d = 2^{10}, \kappa = 10^8\)]{\includegraphics[width=3in]{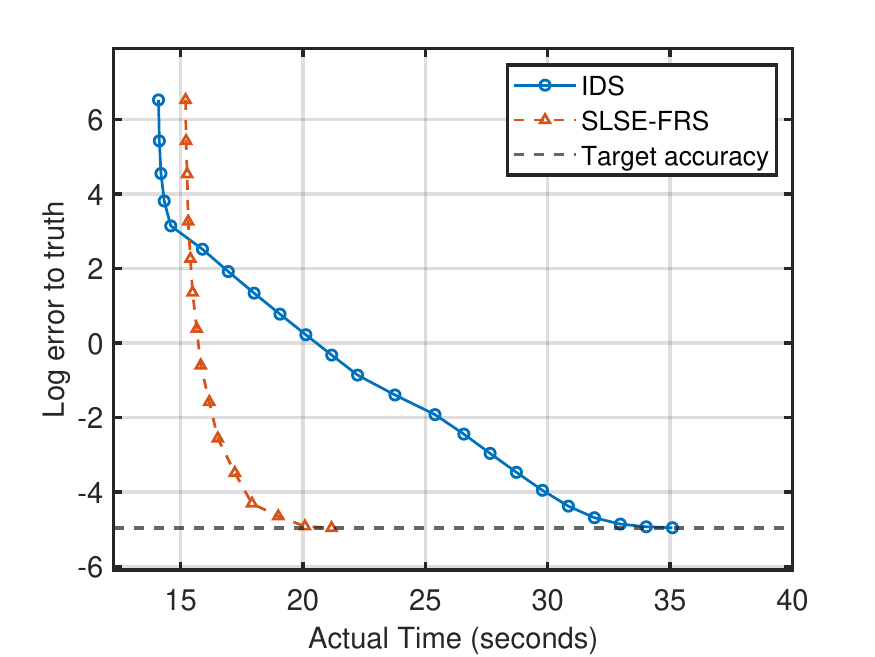}}
	\caption{\(\Delta_t \) versus iterations and actual computing time for SLSE-FRS and IDS.}
	\label{fig5}
	\vskip -0.1in
\end{figure}

Table \ref{table} presents the computing time required for SLSE-FRS and the state-of-the-art methods (i.e., IDS and PCG) to achieve the target precision under the setting of \(N \in \{2^{17}, 2^{18}, 2^{19}, 2^{20}, 2^{22}\}\), $d=2^6$, and $\kappa\in \{10^4,10^8\}$. As shown in this table, SLSE-FRS outperforms IDS and PCG. The computing time of IDS is twice that of SLSE-FRS, while the computing time of PCG is three times that of SLSE-FRS. It is noteworthy that the data sketching overhead of SLSE-FRS is relatively high, and it can be further improved by replacing the SRHT sketching matrix with a more computationally efficient sketching matrix, e.g., the CountSketch matrix.

\begin{table*}[htbp]
	\caption{Actual computing time for SLSE-FRS, IDS and PCG.}
	\label{table}
	\vskip 0.1in
	\centering
	\tabcolsep=0.02\linewidth
	\begin{center}
		\begin{small}
			\begin{sc}
				\begin{tabular}{cccccc}
					\toprule
					 & \multicolumn{4}{c}{$d=2^6$ and \(\kappa = 10^4\)} & $d=2^6$ and \(\kappa = 10^8\) \\
					\midrule
					$N$ & \(2^{17}\) & \(2^{18}\) & \(2^{19}\) & \(2^{20}\) & \(2^{22}\) \\
					\midrule
					IDS  		  & 0.4128 & 0.7316 & 1.6128 & 3.5146 & 18.8237  \\
					PCG           & 0.5968 & 1.0623 & 2.5471 & 5.1601 & 28.7034  \\
					SLSE-FRS  	  & \pmb{0.1791} & \pmb{0.3841} & \pmb{0.7854} & \pmb{1.8701} & \pmb{10.0517}  \\
					\bottomrule
				\end{tabular}
			\end{sc}
		\end{small}
	\end{center}
	\vskip 0.1in
\end{table*}

\subsection{Sketch size tuning}

In this experiment, we consider the possibility of further improving the computational efficiency of SLSE-FRS by tuning its sketch size sequence $m_i$ in large and challenging LS problems. For instance, we let the sample size \(N = 2^{22}\), the feature size \(d = 2^6\), and the condition number \(\kappa = 10^8\), aiming to demonstrate the flexibility and potential of SLSE-FRS for large-scale and difficult problems. If we continue to increase the sketch size by a factor of 2, we will inevitably end up solving several large-scale sketched LS subproblems (e.g., \(m_i = 2^{20}\) or \(2^{21}\)), which remains computationally expensive. Therefore, we opted for the \textbf{tuning technique} as a more flexible approach. We only construct a few large-scale sketched LS subproblems. In this test, SLSE-FRS-tuning only constructs sketched LS subproblems of small sketch sizes \(\{2^{10}, 2^{11}, 2^{12}, 2^{13}, 2^{14}, 2^{15}, 2^{16}\}\) and one larger sketch size \(\{2^{19}\}\). We set \(a_i = 3\) for each sketched LS subproblem.

In Figure \ref{fig3}, we demonstrate the performance of SLSE-FRS, IDS, PCG and SLSE-FRS-tuning. We observe that SLSE-FRS maintains a significant lead in both convergence rate and computational efficiency, thereby improving the state-of-the-art performance of high-precision LS estimators. Combining with the tuning technique, SLSE-FRS-tuning can further reduce the computing time.

We have only found a sequence of sketch sizes that improved the computational efficiency of SLSE-FRS in this experiment, but we currently do not have a good strategy to automatically tune this sequence. The experimental results indicate that the tuning technique has the potential to enhance the computational efficiency of SLSE-FRS, which deserves further research in the future.
\begin{figure}[htbp]
	\centering
	\vskip 0.1in
	\subfigure[\(N = 2^{22}, d = 2^6, \kappa = 10^8\)]{\includegraphics[width=3in]{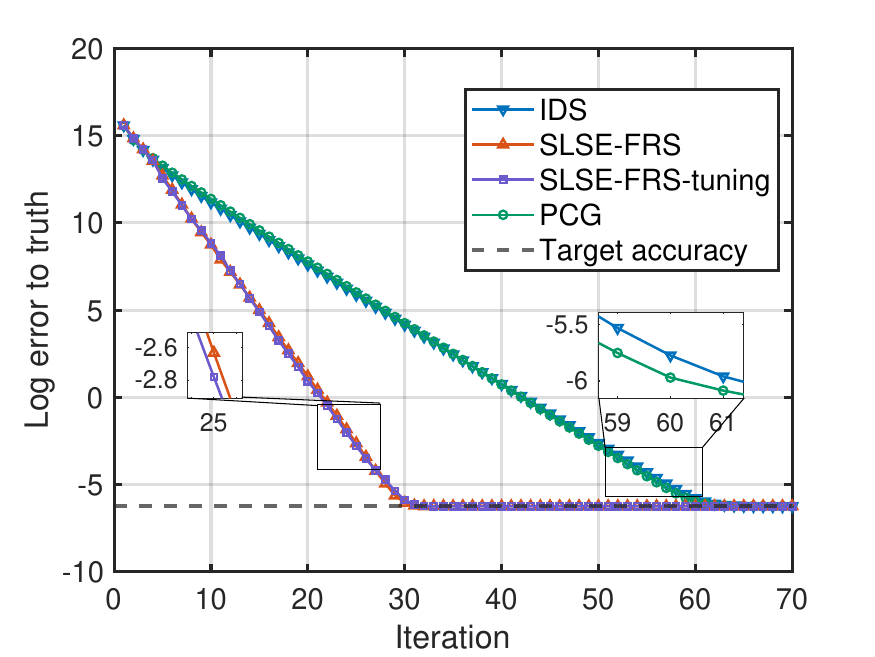}}
	\subfigure[\(N = 2^{22}, d = 2^6, \kappa = 10^8\)]{\includegraphics[width=3in]{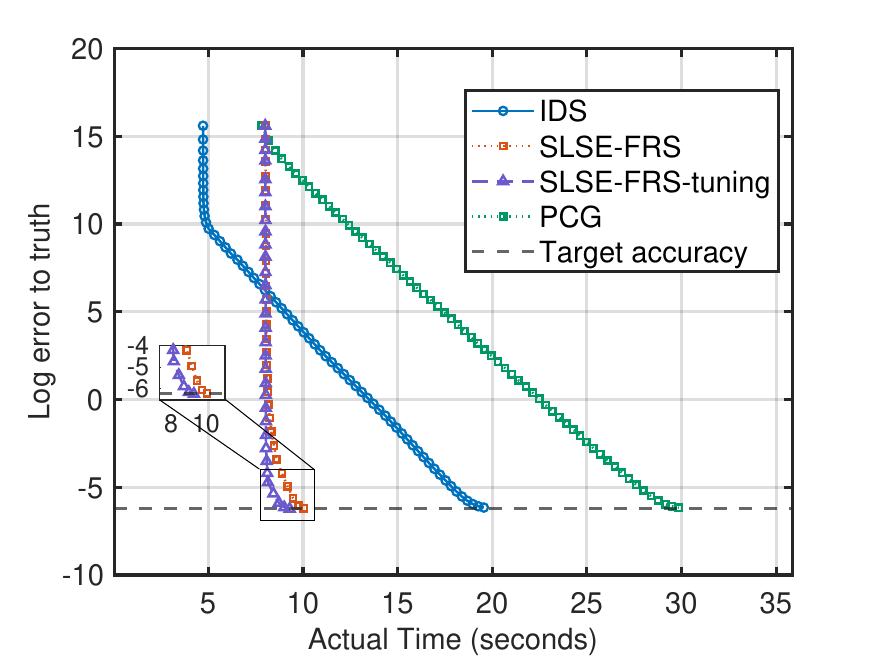}}
	\caption{\(\Delta_t \) versus iterations and actual computing time for SLSE-FRS, SLSE-FRS-tuning, IDS and PCG.}
	\label{fig3}
	\vskip -0.1in
\end{figure}

\subsection{An alternative sketching matrix} \label{altSktchMat}

With the test problem under the setting of \(N = 2^{22}\), \(d = 2^6\), and \(\kappa = 10^8\), we replace SRHT with CountSketch as the sketching matrix. In Figure \ref{fig4}, we have observed a similar convergence rate for both SLSE-FRS with SRHT (SLSE-FRS-SRHT) and SLSE-FRS with CountSketch (SLSE-FRS-CS), while the initialization time of SLSE-FRS-CS is significantly reduced. This experiment indicates that SLSE-FRS, combined with a computationally more efficient sketching matrix, has great potential to further improve its computational efficiency.

\begin{figure}[htbp]
	\centering
	\vskip 0.1in
	\subfigure[\(N = 2^{22}, d = 2^6, \kappa = 10^8\)]{\includegraphics[width=3in]{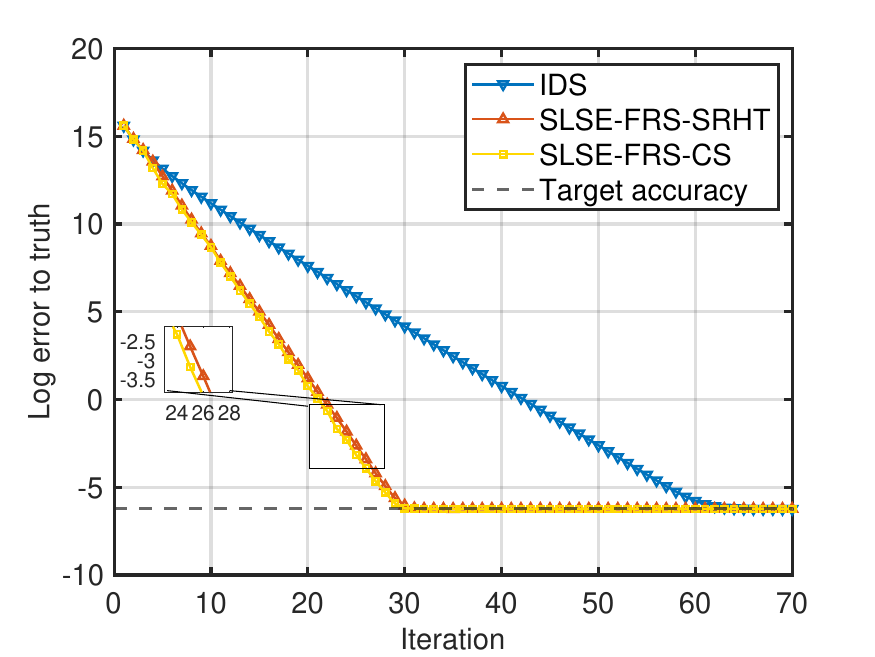}}
	\subfigure[\(N = 2^{22}, d = 2^6, \kappa = 10^8\)]{\includegraphics[width=3in]{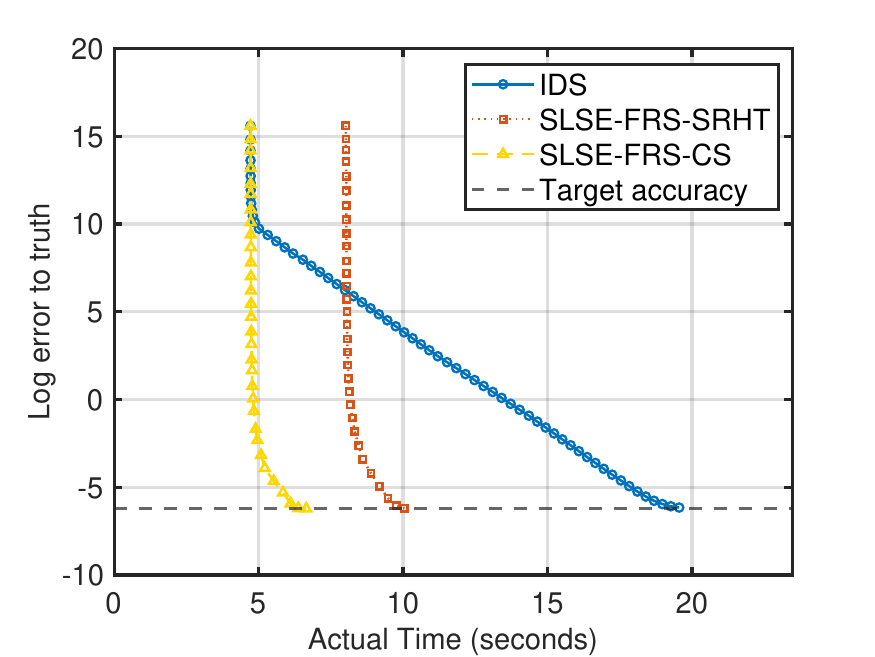}}
	\caption{\(\Delta_t \) versus iterations and actual computing time for SLSE-FRS-SRHT, SLSE-FRS-CS and IDS.}
	\label{fig4}
	\vskip -0.1in
\end{figure}

\section{Conclusions} \label{sec:conclusion}
In this work, we introduced the novel SLSE-FRS framework for large-scale estimation problem of linear statistical models which iteratively constructs and solves sketched LS subproblems with increasing sketch sizes. We investigated the theoretical properties of SLSE-FRS. Numerical experiments demonstrate that SLSE-FRS significantly improves the performance of the state-of-the-art methods. However, several challenges remain open for further investigation.

Although the LS solver M-IHS was used to construct an efficient SLSE-FRS implementation, other alternative LS solvers can also be applied to this novel framework. In future work, more LS solvers can be explored to address a wider range of scenarios.

The appropriate determination of $a_i$ and $m_i$ have significant importance for improving the computational efficiency of SLSE-FRS. Currently, our study on them is still limited, and they deserve further exploration in the future.

In Section \ref{altSktchMat}, we presented the heuristic results of SLSE-FRS with the more efficient CountSketch sketching matrix. Given that most of the computing time is attributed to initialization, it is worthwhile to consider replacing SRHT with CountSketch or other efficient alternatives, and study the related theories of SLSE-FRS with these alternative sketching strategies.

\par

\bibliography{Ref}

\appendix
\section{Appendix}
\subsection{Notations}
The notation used in theoretical analysis is defined as follows. We denote by \( X= U_X D_X V_X^\top \) the reduced singular value decomposition of \( X \), where \( U_X \) is a column orthonormal matrix of size \( N \times d \), \( V_X \) is a orthogonal matrix of size \( d \times d \), and \( D_X \) is a diagonal matrix of size \( d \times d \). The variance denoted by \(\operatorname{Var}[\cdot]\) is taken over the random noise \(\zeta\).

For $\epsilon \in (0,1)$, we define the "good" events
\[
\mathcal{E}_{\epsilon} := \bigcap_{i=1}^{K} \left\{ \left\| U^{\top} S_{i}^{\top} S_{i} U - I_{d} \right\| \leq \epsilon \right\}
\]
and
\[
\hat{\mathcal{E}}_{\epsilon} := \left\{ \left\| U^{\top} \hat{S}^{\top} \hat{S} U - I_{d} \right\| \leq \epsilon \right\}.
\]


\subsection{Proof of Theorem \ref{thm:convergence_theorem1}}\label{B}
In this subsection, we start with some useful lemmas. After that, we present the proof of this theorem.
\begin{lemma}\label{lemma1}
	Suppose the conditions of Theorem \ref{thm:convergence_theorem1} hold. Then for any column orthonormal matrix \( {U} \in \mathbb{R}^{N \times d} \), we have
	\[
	\Pr\left\{ \bigcup_{i=1}^{K} \left\{ \left\| U^{\top} S_i^{\top} S_i U - I_d \right\| > \epsilon \right\} \right\} \leq \delta,
	\]
	and
	\[
	\Pr\left\{ \left\| U^{\top} \hat{S}^{\top} \hat{S} U - I_d \right\| > \epsilon \right\} \leq \delta.
	\]
\end{lemma}

\begin{proof}[\textbf{Proof}]
	For a sufficiently large constant \(c\), Theorem \ref{thm:bigtheorem} ensures the following condition
	\[
	\Pr\left\{\left\|{U}^{\top} {S}_i^{\top} {S}_i {U} - {I}_d \right\| > \epsilon\right\} \leq \frac{\delta}{K}, \quad i = 1, \ldots, K.
	\]
	Since we have \(K\) sketched LS subproblems, by applying the union bound to the above \(K\) events, the first conclusion follows.
	
	The construction of \(\hat{{S}}\) is independent of \({S}_i\) for \(i = 1, \ldots, K\). Using a similar argument, the second conclusion also holds.
\end{proof}

\begin{lemma}\label{lemma:lemma1}
	For any column orthonormal matrix \( {U} \in \mathbb{R}^{N \times d} \), conditioned on the event \( \mathcal{E}_{\epsilon} \), for \( i = 1, \ldots, K \), we have
	\[
	\|({U}^{\top}{S}_{i}^{\top}{S}_{i}{U})^{-1}\| \leq \frac{1}{1 - \epsilon},
	\]
	and conditioned on the event \( \hat{\mathcal{E}}_{\epsilon} \), we have
	\[
	\|({U}^{\top}\hat{{S}}^{\top}\hat{{S}}{U})^{-1}\| \leq \frac{1}{1 - \epsilon}.
	\]
\end{lemma}

\begin{proof}[\textbf{Proof}]
	Conditioned on the event \( \mathcal{E}_{\epsilon} \), for \( i = 1, \ldots, K \), we have
	\[
	\|(U^{\top}S_{i}^{\top}S_{i}U)^{-1}\| = \|(I_{d} + U^{\top}S_{i}^{\top}S_{i}U - I_{d})^{-1}\| \leq \frac{1}{1 - \|U^{\top}S_{i}^{\top}S_{i}U - I_{d}\|} \leq \frac{1}{1 - \epsilon}.
	\]
	The second inequality follows in a similar way.
\end{proof}

\begin{lemma}\label{lemma333}
	For any column orthonormal matrix \( U \in \mathbb{R}^{N \times d} \), conditioned on the event \(\hat{\mathcal{E}}_{\epsilon}\), we have
	\[
	\|I_{d} - \mu(U^{\top}\hat{S}^{\top}\hat{S}U)^{-1}\|
	\leq
	\frac{\epsilon+|\mu-1|}{1-\epsilon}.
	\]
\end{lemma}

\begin{proof}[\textbf{Proof}]
	Conditioned on the event \( \hat{\mathcal{E}}_{\epsilon} \), according to Lemma \ref{lemma:lemma1}, we have \(\|(U^{\top}\hat{S}^{\top}\hat{S}U)^{-1}\| \leq \frac{1}{1 - \epsilon}\). Hence,
	\begin{align*}
		\|I_{d} - \mu(U^{\top}\hat{S}^{\top}\hat{S}U)^{-1}\| &=
		\|(U^{\top}\hat{S}^{\top}\hat{S}U)^{-1}(U_{A}^{\top}\hat{S}^{\top}\hat{S}U - \mu I_{d})\| \\
		&\leq \|(U^{\top}\hat{S}^{\top}\hat{S}U)^{-1}\| \|U^{\top}\hat{S}^{\top}\hat{S}U - \mu I_{d}\| \\
		&\leq \|(U^{\top}\hat{S}^{\top}\hat{S}U)^{-1}\| \left( \|U^{\top}\hat{S}^{\top}\hat{S}U - I_{d}\| + |\mu - 1| \right). \\
		&\leq \frac{\epsilon + |\mu-1|}{1-\epsilon}.
	\end{align*}
\end{proof}

\begin{lemma}\label{lemma:lemma2}
	For any column orthonormal matrix \( U \in \mathbb{R}^{N \times d} \), conditioned on the event \(\hat{\mathcal{E}}_{\epsilon} \cap {\mathcal{E}}_{\epsilon}\), for \( i = 1, \ldots, K \), we have
	\[
	\|I_{d} - \mu(U^{\top}\hat{S}^{\top}\hat{S}U)^{-1}U^{\top}{S}_{i}^{\top}{S}_{i}U\|
	\leq
	\frac{(\mu+1)\epsilon+|\mu-1|}{1-\epsilon}.
	\]
\end{lemma}

\begin{proof}[\textbf{Proof}]
	Conditioned on the event \(\hat{\mathcal{E}}_{\epsilon} \cap {\mathcal{E}}_{\epsilon}\), according Lemma \ref{lemma333}, we have \(
	\|I_{d} - \mu(U^{\top}\hat{S}^{\top}\hat{S}U)^{-1}\|
	\leq
	\frac{\epsilon+|\mu-1|}{1-\epsilon}
	\). Hence, for \( i = 1, \ldots, K \), it follows that
	\begin{align*}
		\|I_{d} - \mu(U^{\top}\hat{S}^{\top}\hat{S}U)^{-1}U^{\top}{S}_{i}^{\top}{S}_{i}U\| &=
		\|I_{d} - \mu(U^{\top}\hat{S}^{\top}\hat{S}U)^{-1} + \mu(U^{\top}\hat{S}^{\top}\hat{S}U)^{-1}(I_{d} - U^{\top}{S}_{i}^{\top}{S}_{i}U)\| \\
		&\leq \|I_{d} - \mu(U^{\top}\hat{S}^{\top}\hat{S}U)^{-1}\| + \mu\|(U^{\top}\hat{S}^{\top}\hat{S}U)^{-1}\|\|I_{d} - U^{\top}{S}_{i}^{\top}{S}_{i}U\| \\
		&\leq \frac{\epsilon + |\mu-1|}{1-\epsilon} + \mu \frac{\epsilon}{1-\epsilon} \\
		&= \frac{(\mu+1)\epsilon+|\mu-1|}{1-\epsilon}.
	\end{align*}
\end{proof}

\begin{proof}[\textbf{Proof of Theorem \ref{thm:convergence_theorem1}}]
	We consider the \(i\)-th sketched LS subproblem. In the case $i = 1$, we set the initial iterates $\beta_{-1}^{1} = \beta_{0}^1 = \beta_0$. In the case \(i > 1\), we set the initial iterates \({\beta_{-1}^{i}} = {\beta_{0}^{i}} = {\beta_{a_{i-1}}^{i-1}}\) based on the iterates from the previous sketched LS subproblem. The update formula is given by
	\[
	\beta_{t + 1}^{i} = \beta_t^{i} - \mu \hat{H}^{-1} \nabla f(\beta_t; S_i X, S_i Y) + \eta ({\beta _t^{i}} - {\beta _{t - 1}^{i}}),
	\]
	which can be expanded as
	\[
	\beta_{t + 1}^{i} = \beta_{t}^{i} - \mu (X^{\top}\hat{S}^{\top}\hat{S}X)^{-1}(S_{i}X)^{\top}(S_{i}X\beta_t^{i}-S_{i}Y) + \eta({\beta _t^{i}} - {\beta_{t - 1}^{i}}).
	\]
	by subtracting the solution \(\tilde{\beta}^{i}\) of the \(i\)-th sketched LS subproblem, we obtain
	\[
	\beta_{t + 1}^{i} - \tilde{\beta}^{i} = \beta_{t}^{i} - \tilde{\beta}^{i} - \mu (X^{\top}\hat{S}^{\top}\hat{S}X)^{-1}(S_{i}X)^{\top}(S_{i}X\beta_t^i-S_{i}Y) + \eta({\beta _t^{i}} - {\tilde{\beta}^{i}}) - \eta({\beta _{t-1}^{i}} - {\tilde{\beta}^{i}}).
	\]
	Since the above iteration is a two-step scheme, we adopt the analysis framework of the heavy ball method \citep{recht2010cs726}, which suggest to consider the following bipartite relation
	\[
	\begin{bmatrix}
		\beta_{t+1}^{i} - \tilde{\beta}^{i} \\
		\beta_t^{i} - \tilde{\beta}^{i}
	\end{bmatrix}
	=
	\begin{bmatrix}
		(1 + \eta)I_d - \mu(X^{\top}\hat{S}^{\top}\hat{S}X)^{-1}(S_{i}X)^{\top}S_{i}X & -\eta I_d \\
		I_d & 0
	\end{bmatrix}
	\begin{bmatrix}
		\beta_t^{i} - \tilde{\beta}^{i} \\
		\beta_{t-1}^{i} - \tilde{\beta}^{i}
	\end{bmatrix}.
	\]
	The above relation leads to
	\[
	\begin{bmatrix}
		D_X V_X^\top(\beta_{t+1}^{i} - \tilde{\beta}^{i}) \\
		D_X V_X^\top(\beta_t^{i} - \tilde{\beta}^{i})
	\end{bmatrix}
	=
	\begin{bmatrix}
		(1 + \eta)D_X V_X^\top - D_X V_X^\top(X^{\top}\hat{S}^{\top}\hat{S}X)^{-1}(S_{i}X)^{\top}S_{i}X & - D_X V_X^\top \\
		D_X V_X^\top & 0
	\end{bmatrix}
	\begin{bmatrix}
		\beta_t^{i} - \tilde{\beta}^{i} \\
		\beta_{t-1}^{i} - \tilde{\beta}^{i}
	\end{bmatrix},
	\]
	which can be simplified as
	
	\begin{equation}\label{contrac}
		\begin{bmatrix}
			D_X V_X^\top(\beta_{t+1}^{i} - \tilde{\beta}^{i}) \\
			D_X V_X^\top(\beta_t^{i} - \tilde{\beta}^{i})
		\end{bmatrix}
		=
		\begin{bmatrix}
			W_i + \eta I_d & -\eta I_d \\
			I_d & 0
		\end{bmatrix}
		\begin{bmatrix}
			D_X V_X^\top(\beta_t^{i} - \tilde{\beta}^{i}) \\
			D_X V_X^\top(\beta_{t-1}^{i} - \tilde{\beta}^{i})
		\end{bmatrix}.
	\end{equation}
    where \(W_i = I_d - \mu (U_{X}^{\top}\hat{S}^{\top}\hat{S}U_{X})^{-1}U_{X}^{\top}S_{i}^{\top}S_{i}U_{X}\). Now, for all $i$, we consider the spectral properties of the iteration matrix
	\[
	L^{(i)} \triangleq \begin{bmatrix}
		W_i + \eta I_d & -\eta I_d \\
		I_d & 0
	\end{bmatrix}.
	\]
	Let the eigenvalue decomposition of \(W_i\) be \(W_i = \tilde{U}_i \Lambda_i \tilde{U}_i^{\top}\), where \(\Lambda_i\) is a real diagonal matrix with its $k$-th diagonal \(\lambda_k^{(i)}\) being the \(k\)-th eigenvalue of \(W_i\). Together with the following permutation \(\Pi\), with its entries defined as
	\[
	\Pi_{i,j} = \begin{cases}
		1 & \text{if } i \text{ is odd and } j = i, \\
		1 & \text{if } i \text{ is even and } j = d + i, \\
		0 & \text{otherwise},
	\end{cases}
	\]
	we can construct a transformation matrix
	\[
	P = \begin{bmatrix}
		\tilde{U}_i & 0 \\
		0 & \tilde{U}_i
	\end{bmatrix}\Pi,
	\]
	which leads to a factorization of the iteration matrix
	\[
	L^{(i)} = P^{-1}
	\begin{bmatrix}
		L_1^{(i)} & 0 & \cdots & 0 \\
		0 & L_2^{(i)} & \cdots & 0 \\
		\vdots & \ddots & \ddots & \vdots \\
		0 & 0 & \cdots & L_d^{(i)}
	\end{bmatrix}P.
	\]
	For \(k = 1, \ldots, d\), the blocks \(L_k^{(i)}\) can be expressed as
	\[
	L_k^{(i)} = \begin{bmatrix}
		\eta + \lambda_k^{(i)} & -\eta \\
		1 & 0
	\end{bmatrix},
	\]
	whose characteristic polynomial is of the form
	\[
	u^2 - (\eta + \lambda_k^{(i)})u + \eta = 0.
	\]
	Based on the condition
	\begin{equation}\label{eigenleq}
		(\eta + \lambda_k^{(i)})^2 \leq 4\eta,
	\end{equation}
	both eigenvalues of \(L_k^{(i)}\) are imaginary and have a magnitude of \(\sqrt{\eta}\). To ensure this condition holds for all \(\lambda_k^{(i)}\), we must determine the appropriate value of \(\eta\).
	
	Due to the conditions $0 < \epsilon < 1/10$ and $|\mu - 1| \leq 1/4$, it follows from Lemma \ref{lemma1} and Lemma \ref{lemma:lemma2} that, with probability at least $1 - 2\delta$, we have $\rho(W_i) \leq \frac{(\mu + 1)\epsilon + |\mu - 1|}{1 - \epsilon} \leq 19/36$. Consequently, for a fixed $k$, the eigenvalues satisfy $|\lambda_k^{(i)}| \leq \rho(W_i) \leq 19/36$.
	
	To minimize the contraction ratio $\sqrt{\eta}$, we solve the inequality \eqref{eigenleq} based on the bound of $\rho(W_i)$, yielding $53/36- \sqrt{17}/3 \leq \eta \leq 53/36+ \sqrt{17}/3$ (refer to Remark \ref{rmkTest}). Then we have the minimum $\eta_{\min} = 53/36 - \sqrt{17}/3 < 1/9$, which leads to \(\rho(L^{(i)}) = \sqrt{\eta_{\min}} < 1/3\)
	
	By repeatedly applying the relation (\ref{contrac}), we have
	\[\begin{bmatrix}
		D_X V_X^\top(\beta_{a_i}^{i} - \tilde{\beta}^{i}) \\
		D_X V_X^\top(\beta_{a_i-1}^{i} - \tilde{\beta}^{i})
	\end{bmatrix}
	=
	\begin{bmatrix}
		W_i + \eta I_d & -\eta I_d \\
		I_d & 0
	\end{bmatrix}^{a_i}
	\begin{bmatrix}
		D_X V_X^\top(\beta_0^{i} - \tilde{\beta}^{i}) \\
		D_X V_X^\top(\beta_{-1}^{i} - \tilde{\beta}^{i})
	\end{bmatrix},
	\]
	which leads to
	\[
	\left\|
	\begin{bmatrix}
		D_X V_X^\top(\beta_{a_i}^{i} - \tilde{\beta}^{i}) \\
		D_X V_X^\top(\beta_{a_i-1}^{i} - \tilde{\beta}^{i})
	\end{bmatrix} \right\|
	\leq
	\| [L^{(i)}]^{a_i} \|
	\left\|
	\begin{bmatrix}
		D_X V_X^\top(\beta_0^{i} - \tilde{\beta}^{i}) \\
		D_X V_X^\top(\beta_{-1}^{i} - \tilde{\beta}^{i})
	\end{bmatrix} \right\|.
	\]
	Thanks to the relation $X = U_X D_X V_X^\top$ and the column orthogonality of $U_X$, the above inequality is equivalent to
	\[
	\left\|
	\begin{bmatrix}
		X(\beta_{a_i}^{i} - \tilde{\beta}^{i}) \\
		X(\beta_{a_i-1}^{i} - \tilde{\beta}^{i})
	\end{bmatrix} \right\|
	\leq
	\| [L^{(i)}]^{a_i} \|
	\left\|
	\begin{bmatrix}
		X(\beta_0^{i} - \tilde{\beta}^{i}) \\
		X(\beta_{-1}^{i} - \tilde{\beta}^{i})
	\end{bmatrix} \right\|,
	\]
	and squaring the above inequality yields
	\[
	\| X(\beta_{a_i}^{i} - \tilde{\beta}^{i}) \|^2 + \| X(\beta_{a_i-1}^{i} - \tilde{\beta}^{i}) \|^2
	\leq
	\| [L^{(i)}]^{a_i} \|^2
	\left(
	\|X(\beta_0^{i} - \tilde{\beta}^{i})\|^2
	+
	\|X(\beta_{-1}^{i} - \tilde{\beta}^{i})\|^2
	\right).
	\]
	Together with \({\beta_{-1}^{i}} = {\beta_{0}^{i}} = {\beta_{a_{i-1}}^{i-1}}\), we obtain
	\[
	\| X(\beta_{a_i}^{i} - \tilde{\beta}^{i}) \|^2
	\leq
	\| [L^{(i)}]^{a_i}\sqrt{2} \|^2
	\|X(\beta_0^{i} - \tilde{\beta}^{i})\|^2.
	\]
	By taking advantage of the Gelfand Formula \citep{kozyakin2009accuracy}, we have
	\begin{align*}
	\lim\limits_{a_i\rightarrow+\infty}
	\|[L^{(i)}]^{a_i}\sqrt{2}\|^{\frac{1}{a_i}}
	&=
	\lim\limits_{a_i\rightarrow+\infty}
	\|[L^{(i)}]^{a_i}\|^{\frac{1}{a_i}}(\sqrt{2})^{\frac{1}{a_i}} \\
	&=
	\lim\limits_{a_i\rightarrow+\infty}
	\|[L^{(i)}]^{a_i}\|^{\frac{1}{a_i}} \\
	&=
	\rho(L^{(i)}).
	\end{align*}
	Therefore, for any $\tau > 0$ satisfying \(\rho(L^{(i)}) + \tau < \frac{1}{3}\), there exists a constant $M_i \geq 1$ such that if $a_i > M_i$,
	\[
	\|[L^{(i)}]^{a_i}\sqrt{2}\|^{\frac{1}{a_i}} - \rho(L^{(i)}) < \tau,
	\]
	or equivalently,
	\begin{align*}
		\|[L^{(i)}]^{a_i}\sqrt{2}\|
		&<
		[\rho(L^{(i)}) + \tau]^{a_i} \\
		&<
		\left(\frac{1}{3}\right)^{a_i}.
	\end{align*}
	Then, it follows that
	\begin{align}
		\| X(\beta_{a_i}^{i} - \tilde{\beta}^{i}) \|
		& \leq
		\left(\frac{1}{3}\right)^{a_i}
		\| X(\beta_{0}^{i} - \tilde{\beta}^{i}) \|. \label{convergetolssolution}
	\end{align}
	
	To investigate the relationship between the true parameter vector $\beta$ and the iterate $\beta_{a_i}^{i}$, we consider the following inequality
	\begin{align}
		\| X(\beta_{a_i}^{i} - {\beta}) \|
		& \leq
		\| X(\beta_{a_i}^{i} - \tilde{\beta}^{i}) \|
		+
		\| X(\tilde{\beta}^{i} - {\beta}) \| \nonumber \\
		& \leq
		\left(\frac{1}{3}\right)^{a_i}
		\|
		X(\beta_{0}^{i} - \tilde{\beta}^{i})
		\|
		+
		\|
		X(\tilde{\beta}^{i} - {\beta})
		\| \nonumber \\
		& \leq
		\left(\frac{1}{3}\right)^{a_i}
		\|
		X(\beta_{0}^{i} - \beta)
		\|
		+
		\left[
		1+\left(\frac{1}{3}\right)^{a_i}
		\right]
		\|
		X(\tilde{\beta^{i}} - {\beta})
		\|,
	\end{align}
	by taking expectation over the random noise \(\zeta\) on both sides, we have
	\[
	\mathbb{E}\| X(\beta_{a_i}^{i} - {\beta}) \|
	\leq
	\left(\frac{1}{3}\right)^{a_i}
	\mathbb{E}\left\|
	X(\beta_{0}^{i} - \beta)
	\right\|
	+
	\left[ 1+\left(\frac{1}{3}\right)^{a_i}
	\right]
	\mathbb{E}\|
	X(\tilde{\beta}^{i} - {\beta})
	\|
	\]
\end{proof}
\begin{remark}\label{rmkTest}
Based on the condition $\rho(W_i) \leq \frac{(\mu + 1)\epsilon + |\mu - 1|}{1 - \epsilon} \leq Z$ with $Z$ being a positive constant, the inequality \eqref{eigenleq} leads to the following bounds
\[
(2 - Z) - 2\sqrt{1 - Z} \leq \eta \leq (2 - Z) + 2\sqrt{1 - Z}.
\]
\end{remark}

\subsection{Proof of Theorem \ref{Thf}}

Before presenting the proof of this theorem, we first state the following two lemmas.

\begin{lemma}\label{finalerror1}
	In Algorithm \ref{Alg-SRHT}, let \(N,m_1\) be powers of 2, \(T^\dagger=\sum_{i=1}^K a_i\),  \(\delta \in (0, 1)\), \(\epsilon \in (0, 1/10)\), \(|\mu-1| \leq 1/4\), and \(\eta = 53/36 - \sqrt{17}/3\). Let
	\[
	\begin{aligned}
		& r \geq c\epsilon^{-2}\left[d + \log\left(\frac{N}{\delta}\right)\right]\log\left(\frac{ed}{\delta}\right), \quad m_1 \textgreater r,\\
	\end{aligned}
	\]
	where \( c > 0 \) is a constant. Define the cumulative inner iteration numbers by
	\[
	\mathcal{T}_i:=\sum_{\ell=1}^i a_\ell,
	\qquad i=1,\ldots,K,
	\]
	with \(\mathcal{T}_0=0\) and \(\mathcal{T}_K=T^\dagger\). Define \(\delta_i:=\mathbb{E}\|X(\tilde{\beta}_i-\beta)\|,\) \(\delta_\star:=\mathbb{E}\|X(\hat{\beta}-\beta)\|.\)
	There exists a constant \({M^{\dagger}} \geq 1\), such that if \(a_i > {M^{\dagger}}\) for all \(i = 1,\ldots, K\), with probability at least \(1-2\delta\), for any \(\beta_0\in\mathbb{R}^d\), the iterate \(\beta_{T^{\dagger}}\) satisfies
	\begin{equation*}
		\mathbb{E}\|X(\beta_{T^{\dagger}}-\beta)\|
		\le  	\left(\frac13\right)^{T^{\dagger}}\mathbb{E}\|X(\beta_0-\beta)\|
		+
		\sum_{i=1}^{K}\left[1+\left(\frac13\right)^{a_i}\right]\left(\frac13\right)^{T^{\dagger}-\mathcal{T}_i}\delta_i.
	\end{equation*}
\end{lemma}

\begin{proof}[\textbf{Proof}]
	In SLSE-FRS, we construct \(K\) sketched LS subproblems and utilize the M-IHS method to take iterations in each sketched LS subproblem.
	
	We begin with the 1st sketched LS subproblem with an initial guess \(\beta_0 \in \mathbb{R}^{d}\), and we set the initial iterates \(\beta_{-1}^{1} = \beta_{0}^1 = \beta_0\). According to Theorem \ref{thm:convergence_theorem1}, there exists a constant \({M_1} \geq 1\), such that if \(a_1 > M_1\), it holds that
	\begin{equation}\label{equation21}
		\mathbb{E}\|X(\beta_{a_1}^1-\beta )\|
		\leq
		\left(\frac{1}{3}\right)^{a_1} \mathbb{E}\|X(\beta_{0}^1-\beta )\|
		+
		\left[
		1+\left(\frac{1}{3}\right)^{a_1}
		\right]
		\delta_{1}.
	\end{equation}
	In the subsequent sketched LS subproblems, we take the following initial iterates, i.e., for \(i = 2,\ldots, K\),
	\[
	{\beta_{-1}^{i}}
	=
	{\beta_0^{i}} = {\beta_{a_{i-1}}^{i-1}}.
	\]
	In the 2nd sketched LS subproblem, there exists a constant \({M_2} \geq 1\), such that if \(a_2 > M_2\), we have
	\begin{align}
		\mathbb{E}\|X(\beta_{a_2}^2-\beta )\|
		& \leq
		\left(\frac{1}{3}\right)^{a_2} \mathbb{E}\|X(\beta_{0}^2-\beta )\|
		+
		\left[
		1+\left(\frac{1}{3}\right)^{a_2}
		\right]
		\delta_{2} \nonumber \\
		& =
		\left(\frac{1}{3}\right)^{a_2} \mathbb{E}\|X(\beta_{a_1}^1-\beta )\|
		+
		\left[
		1+\left(\frac{1}{3}\right)^{a_2}
		\right] \delta_{2} \nonumber \\
		& \leq
		\left(\frac{1}{3}\right)^{{a_1}+{a_2}} \mathbb{E}\|X(\beta_{0}-\beta )\|
		+
		\left[1+\left(\frac{1}{3}\right)^{a_1}\right]\left(\frac{1}{3}\right)^{a_2}\delta_{1}
		+
		\left[1+\left(\frac{1}{3}\right)^{a_2}\right]\delta_{2}, 	
	\end{align}
	where the 3rd inequality follows from inequality (\ref{equation21}).
	
	Then, by deduction, there exists \(K\) constants \(\{{M_i}\}\), let \({M^\dagger} = \max\{M_i \mid i = 1, \ldots, K\}\), if \(a_i > M^\dagger\), we obtain the following relation
	\begin{align}
		\mathbb{E}\|X(\beta_{a_K}^K-\beta )\|
		& \leq
		\left[\prod \limits_{i=1}^{K} \left(\frac{1}{3}\right)^{a_i}\right] \mathbb{E}\|X(\beta_{0}-\beta )\|
		+
		\sum\limits_{i=1}\limits^{K-1}
		\left[1+\left(\frac{1}{3}\right)^{a_i}
		\right]\prod \limits_{j=i+1}^{K}\left(\frac{1}{3}\right)^{a_j}\delta_{i}
		+\left[1+\left(\frac{1}{3}\right)^{a_K}\right]\delta_{K}\nonumber \\
		& = \left(\frac13\right)^{T^{\dagger}}\mathbb{E}\|X(\beta_0-\beta)\|+\sum_{i=1}^{K}\left[1+\left(\frac13\right)^{a_i}\right]\left(\frac13\right)^{T^{\dagger}-\mathcal{T}_i}\delta_i, \label{22}
	\end{align}
	which is the claimed conclusion.
\end{proof}

The following lemma derives the convergence behavior of the 2nd stage of SLSE-FRS.

\begin{lemma}\label{M-IHS converging rate}
	Suppose that the conditions of Theorem \ref{Thf} hold, there exists a constant \({M^\ast} \geq 1\), such that if \(T > T^\dagger+ {M^\ast}\), it follows that
	\[
	\mathbb{E}\|X(\beta_T- \hat{\beta})\|
	\leq
	\left(\frac{1}{3}\right)^{T-T^\dagger} \mathbb{E}\|X(\beta_{T^\dagger}- \hat{\beta})\|.
	\]
\end{lemma}

\begin{proof}[\textbf{Proof}]
	After \(T^\dagger\) iterations at the 1st stage of SLSE-FRS, for \(t = T^\dagger+1, \ldots, T\), we employ M-IHS to solve the original full-scale LS problem. The update formula is given by
	\[
	\beta_{t + 1} = \beta_t - \mu \hat{H}^{-1} \nabla f(\beta_t; X, Y) + \eta (\beta_t - \beta_{t - 1})
	\]
	with \(\hat{H} = X^{\top}\hat{S}^{\top}\hat{S}X\). Subtracting the exact OLS estimator \(\hat{\beta}\), we obtain
	\[
	\beta_{t + 1} - \hat{\beta} = \beta_t - \hat{\beta} - \mu (X^{\top}\hat{S}^{\top}\hat{S}X)^{-1}X^{\top}(X\beta_t - Y) + \eta (\beta_t - \hat{\beta}) - \eta (\beta_{t-1} - \hat{\beta}).
	\]
	The above iteration leads to the following bipartite relation
	\[
	\begin{bmatrix}
		\beta_{t+1} - \hat{\beta} \\
		\beta_t - \hat{\beta}
	\end{bmatrix}
	=
	\begin{bmatrix}
		(1 + \eta)I_d - \mu(X^{\top}\hat{S}^{\top}\hat{S}X)^{-1}X^{\top}X & -\eta I_d \\
		I_d & 0
	\end{bmatrix}
	\begin{bmatrix}
		\beta_t - \hat{\beta} \\
		\beta_{t-1} - \hat{\beta}
	\end{bmatrix}.
	\]
	This error recurrence has a similar form as that considered in the proof of \cref{thm:convergence_theorem1}. Therefore, with the same analysis strategy, if \(T - T^\dagger > M^\ast\), we have
	\[
	\mathbb{E}\|X(\beta_T- \hat{\beta})\|
	\leq
	\left(\frac{1}{3}\right)^{T-T^\dagger} \mathbb{E}\|X(\beta_{T^\dagger}- \hat{\beta})\|.
	\]
\end{proof}

\begin{proof}[\textbf{Proof of Theorem \ref{Thf}}]
	If \(s=T-T^{\dagger}>M^* \), due to Lemma \ref{M-IHS converging rate}, we have
	\[
	\mathbb{E}\|X(\beta_T- \hat{\beta})\|
	\leq
	\left(\frac{1}{3}\right)^{s} \mathbb{E}\|X(\beta_{T^\dagger}- \hat{\beta})\|.
	\]
	Then it follows that
	\begin{align*}
		\mathbb{E}\|X(\beta_T- {\beta})\|
		& \leq
		\mathbb{E}\|X(\beta_T- \hat{\beta})\|
		+
		\mathbb{E}\|X(\hat{\beta}- {\beta})\| \nonumber \\
		& \leq
		\left(\frac{1}{3}\right)^{s} \mathbb{E}\|X(\beta_{T^\dagger}- \hat{\beta})\|
		+
		\delta_\star.
	\end{align*}
	Let \(M = \max\left({M^\ast}, {M^{\dagger}} \right)\), if  \(\min\{s, a_i\} > M \) holds for all \(i=1,\ldots,K\), then it follows that
	\begin{align} \label{mainpf_1}
		\mathbb{E}\|X(\beta_T- {\beta})\|
		& \leq
		\left(\frac{1}{3}\right)^{T}
		\mathbb{E}\|X(\beta_{0}-\beta )\|
		+
		\left(\frac{1}{3}\right)^{s}B_K
		+
		\left[1+\left(\frac13\right)^s
		\right]\delta_\star,
	\end{align}
	where
	\begin{equation*}
		B_K=
		\sum_{i=1}^{K}
		\left[
		1+\left(\frac13\right)^{a_i}
		\right]
		\left(\frac13\right)^{T^\dagger-\mathcal{T}_i}
		\delta_i.
	\end{equation*}
	Moreover, for each \(i=1,\ldots,K\), since
	\begin{equation*}
		a_i+T^\dagger-\mathcal{T}_i
		=
		T^\dagger-\mathcal{T}_{i-1},
	\end{equation*}
	we have
	\begin{equation*}
		\left[
		1+\left(\frac13\right)^{a_i}
		\right]
		\left(\frac13\right)^{T^\dagger-\mathcal{T}_i}
		=
		\left(\frac13\right)^{T^\dagger-\mathcal{T}_i}
		+
		\left(\frac13\right)^{T^\dagger-\mathcal{T}_{i-1}} .
	\end{equation*}
	Together with the condition that \(\delta_i\le C\delta_\star, i=1,\ldots,K\), it follows that
	\begin{align*}
		B_K
		&\le
		C\delta_\star
		\left[
		\sum_{i=1}^{K}
		\left(\frac13\right)^{T^\dagger-\mathcal{T}_i}
		+
		\sum_{i=1}^{K}
		\left(\frac13\right)^{T^\dagger-\mathcal{T}_{i-1}}
		\right].
	\end{align*}
	We next bound the two sums separately. Since \(a_i>M\) for all
	\(i=1,\ldots,K\), it follows that
	\begin{equation*}
		T^\dagger-\mathcal{T}_i
		=
		\sum_{\ell=i+1}^{K}a_\ell
		>
		(K-i)M .
	\end{equation*}
	Therefore,
	\begin{equation}\label{mainpf_2}
		\sum_{i=1}^{K}
		\left(\frac13\right)^{T^\dagger-\mathcal{T}_i}
		\le
		\sum_{j=0}^{K-1}
		\left(\frac13\right)^{jM}
		\le
		\frac{1}{1-\left(\frac13\right)^M}.
	\end{equation}
	Similarly,
	\begin{equation*}
		T^\dagger-\mathcal{T}_{i-1}
		=
		\sum_{\ell=i}^{K}a_\ell
		>
		(K-i+1)M,
	\end{equation*}
	which leads to
	\begin{equation}\label{mainpf_3}
		\sum_{i=1}^{K}
		\left(\frac13\right)^{T^\dagger-\mathcal{T}_{i-1}}
		\le
		\sum_{j=1}^{K}
		\left(\frac13\right)^{jM}
		\le
		\frac{\left(\frac13\right)^M}{1-\left(\frac13\right)^M}.
	\end{equation}
	Since \(M\ge 1\), we have \(\left(1/3\right)^M\le 1/3\), combining \eqref{mainpf_2} and \eqref{mainpf_3} yields
	\begin{equation*}
		B_K
		\le
		C\delta_\star
		\left[
		\frac{1}{1-\left(\frac13\right)^M}
		+
		\frac{\left(\frac13\right)^M}{1-\left(\frac13\right)^M}
		\right]
		=
		C\delta_\star
		\frac{1+\left(\frac13\right)^M}
		{1-\left(\frac13\right)^M}
		\le
		2C\delta_\star.
	\end{equation*}
	Substituting this estimate into \eqref{mainpf_1} gives
	\begin{align*}
		\mathbb{E}\|X(\beta_T-\beta)\|
		&\le
		\left(\frac13\right)^T
		\mathbb{E}\|X(\beta_0-\beta)\|
		+
		2C\left(\frac13\right)^s\delta_\star
		+
		\left[
		1+\left(\frac13\right)^s
		\right]\delta_\star  \\
		&=
		\left(\frac13\right)^T
		\mathbb{E}\|X(\beta_0-\beta)\|
		+
		\left[
		1+
		\left(\frac13\right)^s(1+2C)
		\right]\delta_\star .
	\end{align*}
\end{proof}

\subsection{Lower bound of \(a_1\)}\label{appendixa1}
The following theorem provides an lower bound of the iteration count $a_1$ for the 1st sketched LS subproblem.
\begin{theorem}\label{a1}
	In SLSE-FRS, given a prescribed precision \(\omega \in (0,1)\), the number of iterations \(a_1\) needed for the 1st sketched LS subproblem to fulfill (\ref{14}) satisfies
	\[
	a_1 \geq \log_3 \left( \frac{\mathbb{E}\|X({\beta}_0 - \tilde{\beta}_1)\|}{\omega \mathbb{E}\|X(\tilde{\beta}_1 - \beta)\|} \right).
	\]
\end{theorem}

\begin{proof}[\textbf{Proof}]
	To achieve the prescribed precision $\omega$, i.e., (\ref{14}) being fulfilled, the iteration terminates when the following condition is met, i.e.,
	\[
	\mathbb{E}\|X({\beta}_{a_1}^1 - \tilde{\beta}^1)\| + \mathbb{E}\|X({\tilde{\beta}^1} - \beta)\| \leq (1+\omega)\mathbb{E}\|X({\tilde{\beta}^1} - \beta)\|.
	\]
	In fact, we need
	\[
	\mathbb{E}\|X({\beta}_{a_1}^1 - \tilde{\beta}^1)\|
	\leq
	\omega\mathbb{E}\|X({\tilde{\beta}^1} - \beta)\|.
	\]
	According to (\ref{convergetolssolution}) in the proof of Theorem \ref{thm:convergence_theorem1}, it holds that
	\[
	\mathbb{E}\|X({\beta}_{a_1}^1 - \tilde{\beta}^1)\|
	\leq
	\left(\frac{1}{3}\right)^{a_1}\mathbb{E}\|X({\beta}_{0} - \tilde{\beta}^1)\| .
	\]
	By requiring
	\[
	\left(\frac{1}{3}\right)^{a_1} \mathbb{E}\|X({\beta}_{0} - \tilde{\beta}^1)\|
	\leq
	\omega\mathbb{E}\|X({\tilde{\beta}^1} - \beta)\|,
	\]
	and taking logarithms on both sides of the above inequality, then the claimed result is obtained.
\end{proof}

\subsection{Proof of Theorem \ref{lowerBoundai}} \label{appendixa5}
In this subsection, we provide the proof of the lower bound of $a_i$ for the $i$-th sketched LS subproblem under some additional assumptions specified below. We first introduce the following auxiliary lemma.

Lemma \ref{lemma:lemma3} is a reformulation of Theorem 2.4 from \citep{dobriban2019asymptotics}. We use it to explore the loss of accuracy of the \(i\)-th sketched LS estimator against the OLS estimator when using \textit{Sketch-and-Solve} methods with SRHT sketching matrix of different sizes.

\begin{lemma}\label{lemma:lemma3}
	For the \(i\)-th sketched LS subproblem, let \( S \) be an \( N \times N \) subsampled randomized Hadamard matrix. Suppose also that \( X \) is an \( N \times d \) deterministic matrix whose empirical spectral distribution converges weakly to some fixed probability distribution with compact support bounded away from the origin. Then as \( N \) tends to infinity, while \( d/N \rightarrow \gamma \in (0,1) \), \( m_i/N \rightarrow \xi \in (\gamma, 1) \), the \(i\)-th relative prediction efficiency $\mathrm{PE}(i)$ has the limit
	\[
	\mathrm{PE}(i) = \frac{\mathbb{E}\|X\tilde{\beta}^i - X\beta\|^2}{\mathbb{E}\|X\hat{\beta} - X\beta\|^2} = \frac{1 - \gamma}{\xi - \gamma}.
	\]
\end{lemma}
\begin{remark}\label{remarkA7}
	Let \( g(\gamma,\xi) = (1 - \gamma)/(\xi - \gamma) \), and it is monotonically increasing with respect to \(\gamma\) and monotonically decreasing with respect to \(\xi\). According to Lemma \ref{lemma:lemma3}, when \( \xi > m_i / N > d/N > \gamma \), it follows that
	\begin{align*}
		\text{PE}(i)
		& \leq \frac{1- d/N}{m_i/N - d/N} \\
		& = \frac{N - d}{m_i - d},
	\end{align*}
	or equivalently,
	\[
	\mathbb{E}\|X\tilde{\beta}^i - X\beta\|^2 \leq \frac{N - d}{m_i - d}{\mathbb{E}\|X\hat{\beta} - X\beta\|^2}.
	\]
	Furthermore, if the condition below holds
	\[
	\operatorname{Var}(\|X\tilde{\beta}^i - X\beta\|) \geq \operatorname{Var}(\sqrt{\frac{N - d}{m_i - d}}\|X\hat{\beta} - X\beta\|),
	\]
	we obtain
	\[
	\mathbb{E}\|X\tilde{\beta}^i - X\beta\| \leq \sqrt{\frac{N - d}{m_i - d}}{\mathbb{E}\|X\hat{\beta} - X\beta\|}.
	\]
	Therefore, the subsequent discussions are based on the following conditions:
	\begin{itemize}
		\item \(\xi > m_i / N > d/N > \gamma\);   \\
		\item \(\operatorname{Var}(\|X\tilde{\beta}^i - X\beta\|) \geq \operatorname{Var}(\sqrt{(N - d)/(m_i - d)}\|X\hat{\beta} - X\beta\|)\).
	\end{itemize}
\end{remark}

\begin{proof}[\textbf{Proof of Theorem \ref{lowerBoundai}}]
	In the \(i\)-th sketched LS subproblem, to achieve the specified precision $\omega$, i.e., (\ref{14}) being satisfied, the iteration stops when the following condition meets
	\[
	\mathbb{E}\|X({\beta}_{a_i}^i - \tilde{\beta}^i)\| + \mathbb{E}\|X({\tilde{\beta}^i} - \beta)\| \leq (1+\omega)\mathbb{E}\|X({\tilde{\beta}^i} - \beta)\|.
	\]
	In fact, we need
	\begin{equation}\label{23}
		\mathbb{E}\|X({\beta}_{a_i}^i - \tilde{\beta}^i)\|
		\leq
		\omega\mathbb{E}\|X({\tilde{\beta}^i} - \beta)\|.
	\end{equation}
	According to (\ref{convergetolssolution}) in the proof of Theorem \ref{thm:convergence_theorem1}, it follows that
	\[
	\mathbb{E}\|X({\beta}_{a_i}^i - \tilde{\beta}^i)\|
	\leq
	\left(\frac{1}{3}\right)^{a_i} \mathbb{E}
	\|
	X(\beta_{0}^{i} - \tilde{\beta}^{i})
	\|.
	\]
	Based on the relation between the initial iterate \(\beta_{0}^{i} = {\beta_{a_{i-1}}^{i-1}}\) and the exact solution of the \(i\)-th sketched LS subproblem, we have
	\begin{equation}\label{25}
		\mathbb{E}\|X({\beta_{a_{i-1}}^{i-1}} - \tilde{\beta}^i)\|
		\leq
		\mathbb{E}\|X({\beta_{a_{i-1}}^{i-1}} - \tilde{\beta}^{i-1})\|
		+
		\mathbb{E}\|X({\tilde{\beta}^{i-1}} - {\beta})\|
		+
		\mathbb{E}\|X({\tilde{\beta}^{i}} - {\beta})\|.
	\end{equation}
	In a similar fashion of Lemma \ref{lemma:lemma3} and Remark \ref{remarkA7}, as \( N \) tends to infinity, we let \( d/N \rightarrow \gamma \in (0,1) \),  \( m_i/N \rightarrow \xi_i \in (\gamma, 1) \).
	When the following condition holds
	\[
	\xi_i > \frac{m_i}{N} > \frac{d}{N} > \gamma,
	\]
	we obtain
	\[
	 r(i-1,i) = \sqrt{\frac{m_{i} - d}{m_{i -1} - d}} \geq \frac{\mathbb{E}\|X({\tilde{\beta}^{i-1}} - {\beta})\|}{\mathbb{E}\|X({\tilde{\beta}^{i}} - {\beta})\|},
	\]
	which leads to
	\begin{equation}\label{26}
		\mathbb{E}\|X({\tilde{\beta}^{i-1}} - {\beta})\|
		\leq
		r(i-1,i) \mathbb{E}\|X({\tilde{\beta}^{i}} - {\beta})\|.
	\end{equation}
	Since the iteration of the \((i-1)\)-th sketched LS subproblem also stops when the precision \(\omega\) is achieved, we get
	\begin{equation}\label{27}
		\mathbb{E}\|X({\beta}_{a_{i-1}}^{i-1} - \tilde{\beta}^{i-1})\|
		\leq
		\omega\mathbb{E}\|X({\tilde{\beta}^{i-1}} - \beta)\|.
	\end{equation}
	By substituting (\ref{26}) and (\ref{27}) into (\ref{25}), it follows that
	\begin{equation}\label{last}
		\mathbb{E}\|X({\beta_{a_{i-1}}^{i-1}} - \tilde{\beta}^i)\|
		\leq
		[(1+\omega)r(i-1,i)+1]
		\mathbb{E}\|X({\tilde{\beta}^{i}} - {\beta})\|
	\end{equation}
	The inequality (\ref{23}) holds if
	\[
	\left(\frac{1}{3}\right)^{a_i} [(1+\omega)r(i-1,i)+1]
	\mathbb{E}\|X({\tilde{\beta}^{i}} - {\beta})\|
	\leq
	\omega\mathbb{E}\|X({\tilde{\beta}^i} - \beta)\|.
	\]
	By taking logarithms on both sides of the above inequality, it results in
	\begin{align*}
		a_i
		& \geq
		\log_3 \left[ \frac{(1+\omega)r(i-1, i)+1}{\omega} \right].
	\end{align*}
	Since \(\omega \in (0,1)\) and \([(1+\omega)r(i-1, i)+1] > 1 \), it reads that \( a_i > 0 \).
\end{proof}

\subsection{Proof of Theorem \ref{cmpxtyAlgSLSE}} \label{cmpxtyAlgSLSEPrf}

The proof of this theorem below is given in the asymptotic sense, that is, in the sense of $N\rightarrow +\infty$.
\begin{proof}[\textbf{Proof}]
  The conditions \(m_{i+1}/m_i=2\) and $m_K=N/2$ lead to $m_i=2^{-(K+1-i)}N$, together with $d/N\rightarrow\gamma\in(0,2^{-K})$ as $N\rightarrow+\infty$, it follows that $r(i-1,i)=\sqrt{(1-2^{K+2-i}\gamma)/(2-2^{K+2-i}\gamma)}$ ($i=2,\ldots,K$). Since $2^{K+2-i}\gamma<1$, the function $r(x)=\sqrt{(1-x)/(2-x)}$ monotonically decreases with respect to $x\in (0,1)$, and the function $\ell(r)=\log_3\{[1+(1+\omega)r]/\omega\}$ monotonically increases with respect to $r>0$, it holds that $a_i=\log_3\{[1+(1+\omega)r(i-1,i)]/\omega\}$ monotonically increases with repect to $i=2,\ldots,K$, i.e., $a_2<\ldots<a_K$. Due to the above results, we have
  \begin{align*}
    a_2 = & \log_3\{[1+(1+\omega)\sqrt{(1-2^{K}\gamma)/(2-2^{K}\gamma)}]/\omega\} \\
    > & \log_3 (1/\omega)
  \end{align*}
  and
  \begin{align*}
    a_K = & \log_3\{[1+(1+\omega)\sqrt{(1-2^{2}\gamma)/(2-2^{2}\gamma)}]/\omega\} \\
    < & \log_3\{[1+(1+\omega)/\sqrt{2}]/\omega\}=\alpha,
  \end{align*}
  which lead to the fact $\log_3(1/\omega)<a_i<\alpha$ ($i=2,\ldots,K$). For $i=1$, according to Appendix \ref{appendixa1}, taking the lower bound $a_1 = \log_3 [ (1/\omega )(\mathbb{E}\|X({\beta}_0 - \tilde{\beta}^1)\|/\mathbb{E}\|X(\tilde{\beta}^1 - \beta)\|)]$, together with the condition $1 < \mathbb{E}\|X({\beta}_0 - \tilde{\beta}_1)\| / \mathbb{E}\|X(\tilde{\beta}_1 - \beta)\| < 1+(1+\omega)/\sqrt{2}$, it follows that $\log_3(1/\omega)<a_1<\alpha$.

  Now, we consider the dominant costs of all stages of Algorithm \ref{Alg-SRHT}. At the beginning, the main workload in the initialization stage includes the construction of all of the required data, that is, $(S_0X,S_0Y)$, $(S_iX,S_iY)$ ($i=1,\ldots,K$), and $\hat{H}$. According to the discussion in Section \ref{sec:effcntSLSE}, the cost of this stage is dominated by applying the Hadamard transform to a matrix of size \(N\times d\), which is dominated by \(Nd\log_2 N\) FLOPS.

  Next, we will consider the two stages of iteration. Since Algorithm \ref{Alg-SRHT} terminates when a noise level estimator is attained, which means that the 1st term of the upper bound in (\ref{final_precision}) has decreased to \(\sigma\)-level. Due to the fact that the initial error $\mathbb{E}\|X(\beta_0 - \beta)\|$ is a constant, it is equivalent to require the factor $(1/3)^T$ reduces to \(\sigma\)-level, i.e., $(1/3)^T\le\sigma$. Therefore, the total count $T$ of iteration is bounded below by $T\ge \log_3(1/\sigma)$. To minimize the cost of Algorithm \ref{Alg-SRHT} in this discussion, we let $T = \log_3(1/\sigma)$.

  In the 1st stage of iteration, Algorithm \ref{Alg-SRHT} takes \(a_i\) iterations in the \(i\)-th sketched LS subproblem, the total \(K\) sketched subproblems requires \(\sum_{i=1}^{K} a_i[(4d+1)m_i+2d^2+5d]\) FLOPs. Thanks to the fact $\sum_{i=1}^{K} m_i = (1-2^{-K})N$, together with the upper bound $a_i<\alpha$, it holds that the dominant cost in this stage is bounded by $4\alpha Nd$.

  In the 2nd stage of iteration, Algorithm \ref{Alg-SRHT} takes full-scale iterations, and the iteration count reads
  \begin{align*}
    T-T^{\dagger} = & \log_3(1/\sigma) - \sum_{i=1}^{K} a_i \\
    < & \log_3(1/\sigma) - K\log_3(1/\omega) = \log_3(\omega^K/\sigma).
  \end{align*}
  The main cost of each full-scale iteration comes from the computation of gradients $\nabla f({\beta}_{t}; {X}, {Y})$, which is dominated by $4Nd$. Together with the above bound of $T-T^{\dagger}$, it follows that the dominant cost in this stage is bounded by $4 [\log_3(\omega^{K}/\sigma)] Nd > 0$ with $\omega>\sigma^{1/K}$.
\end{proof}

\subsection{Implementation details}\label{details}

According to \citep{epperly2024fast}, the inverse of the sketched Hessian matrix applied to a vector is implemented based on the QR factorization. Due to the inefficiency of the MATLAB function \textsf{fwht}, we adopt a C$++$-FWHT library whenever the application of the Hadamard transform is needed in the implementation of SLSE-FRS, IDS, and PCG, which can be found at \url{https://github.com/jeffeverett/hadamard-transform}. No other additional libraries were used in our numerical experiments.

\end{document}